\documentclass{article}

\usepackage{amsmath,amsthm,amscd,amssymb,amsfonts}
\usepackage{mathtools} 
\usepackage{tikz} 
\usepackage{tikz-cd} 
\usepackage[normalem]{ulem}
\usepackage{float}
\usepackage{graphicx} 
\usepackage{longtable}
\usepackage{enumerate} 
\usepackage[hidelinks]{hyperref} 
\usepackage{xcolor, soul} 
\usepackage{adjustbox} 
\usepackage{multicol} 
\usepackage{array} 
\usepackage{ulem} 
\usepackage{adjustbox} 
\usepackage[table]{xcolor}
\usepackage{algorithm}
\usepackage{algorithmic}
\usepackage{amsmath, amssymb}
\usepackage{tcolorbox}
\usepackage{booktabs}
\usepackage{caption}
\usepackage{geometry}
\usepackage{tabularx}
\usepackage{multirow}
\usepackage{ulem}
\usepackage{booktabs}

\newtheorem{theorem}{Theorem}[section]
\newtheorem{proposition}[theorem]{Proposition}

\theoremstyle{definition}
\newtheorem{definition}[theorem]{Definition}

\theoremstyle{remark}
\newtheorem{remark}[theorem]{Remark}

\usepackage{titling}



\title{Target-Aware State-Adaptive \(p\)-Dirichlet Graph Neural Regression for Non-Invasive Body-Composition Estimation}
\author{
Nadejda Drenska$^{1,\dagger,*}$,
Matthew Lemoine$^{1,\dagger}$,
Gowri Priya Sunkara$^{1,\dagger}$,
Yu Wang$^{1,\dagger,*}$,\\
Sri Lakshmi Sravani Devarakonda$^{2}$,
and Steven B. Heymsfield$^{2,*}$
}
\date{}

\newcommand{\R}{\mathbb{R}}
\newcommand{\E}{\mathcal{E}}
\newcommand{\ip}[2]{\left\langle #1,#2\right\rangle_F}

\begin{document}

\maketitle

\vspace{-1.5em}
\begin{center}
\small
$^{1}$Department of Mathematics, Louisiana State University,
Baton Rouge, LA 70803, USA.\\
$^{2}$Pennington Biomedical Research Center(PBRC),
Baton Rouge, LA 70808, USA.\\[0.4em]
$^{\dagger}$These authors contributed equally to this work and are
listed alphabetically.\\
$^{*}$Correspondence:
\href{mailto:ndrenska@lsu.edu}{ndrenska@lsu.edu} (N.D.),
\href{mailto:yuwang@lsu.edu}{yuwang@lsu.edu} (Y.W.), and
\href{mailto:steven.heymsfield@pbrc.edu}
{steven.heymsfield@pbrc.edu} (S.B.H.).
\end{center}
\normalsize

\begin{abstract}

Accurate estimation of body-composition outcomes, including body fat percentage (BFP), bone mineral density (BMD), and appendicular lean mass (ALM), is important for evaluating metabolic, skeletal, and muscular health. However, direct assessment using dual-energy X-ray absorptiometry (DXA) requires specialized equipment and involves ionizing radiation. We propose a target-aware, state-adaptive $p$-Dirichlet energy-flow graph neural regression ($p$SADE-GNR) framework for estimating these outcomes from non-invasive anthropometric measurements. A neural encoder maps participant representations to hidden states, which are propagated over an outcome-specific participant-similarity graph using a state-adaptive forward-Euler discretization of the graph $p$-Dirichlet energy flow. Residual averaging and a linear regression layer then produce the predictions. For target-aware variants, the graph distance weights each original or latent coordinate by its normalized absolute training-fold correlation with the outcome.
Using clinical data from the Pennington Biomedical Research Center and five-fold cross-validation, the correlation-weighted model constructed from the original standardized measurements achieves the lowest root mean squared error among the evaluated configurations in all nine primary outcome–cohort combinations. It also obtained lower error than previously reported external support vector regression or least-squares support vector regression reference values in eight of nine comparisons. Autoencoder, variational-autoencoder, and Gaussian-mixture variational-autoencoder representations generally did not improve prediction of the primary outcomes or reduce computational cost. However, the correlation-weighted GMVAE model achieved the lowest mean error in all three cohorts in an exploratory age-prediction setting that included ALM, BMD, and BFP as predictors. These findings support target-aware, state-adaptive $p$-Dirichlet energy-flow graph neural regression for non-invasive body-composition estimation.

\end{abstract}

\section{Introduction}\label{sec:intro}
Appendicular Lean Mass (ALM), Body Fat Percentage (BFP), and Bone Mineral
Density (BMD) describe complementary aspects of muscular, adipose, and
skeletal status. However, their clinical and research utility depends on the accuracy with which they are estimated. A prediction model may reproduce the overall population trend
while still making absolute errors that obscure a modest longitudinal change
or alter the interpretation of a value near a population- or clinically
relevant reference boundary. For this reason, body-composition estimation
should be evaluated not only through association measures, but also through
errors expressed in the original units: kilograms for ALM, percentage points
for BFP, and $\mathrm{g/cm^2}$ for BMD. Performance should also be examined
across population subgroups, because a low aggregate error can conceal
systematic differences between groups. The estimates considered here are
intended to support quantitative assessment and research; they are not
presented as standalone diagnoses or replacements for clinical judgment
\cite{rheumatoid_arthritis,HIND2018429,osteoporosis}.

Dual-energy X-ray absorptiometry (DXA) is widely used to obtain reference
measurements of lean tissue, fat tissue, and bone-related measurements, but it
requires dedicated equipment, trained personnel, and an in-person acquisition
protocol \cite{cdc_dexa,HIND2018429}. These requirements motivate the study of
lower-burden predictors, including anthropometric and demographic variables,
bioimpedance measurements, photographs, and three-dimensional (3D) optical
body scans. Machine-learning models based on compact anthropometric, demographic,
bioimpedance, clinical, or genomic variables have been developed for lean
mass, ALM, BFP, and BMD prediction
\cite{birk2025equations,cichosz2021precise,fan2022bodyfat,ALM_prediction_paper,shioji2017bmd,ucar2021bfp,wu2021bmd}. A complementary literature uses
photographs and 3D body shape to estimate total, regional, and
outcome-specific body composition
\cite{feng2026voxel,marazzato2024alm,pleuss2019machinelearning,qiao2024prediction,
tian2025cnn}. Collectively, these studies establish that
accessible measurements and external body shape contain substantial
predictive information. However, most existing approaches either process each
participant as an independent feature vector or learn directly from
high-dimensional images, meshes, or voxel representations. They do not
resolve a central question in our setting: how should a compact anthropometric
profile define similarity between participants when the prediction target
changes?

A participant-similarity graph provides a natural framework for this
question: participants are represented as nodes, and edges encode proximity
in the measurement space. Such models are motivated by the smoothness
assumption that participants who are similar in outcome-relevant
characteristics tend to exhibit similar responses \cite{pgnn, pai2018patient}. The
\(p\)-Dirichlet energy provides a mechanism for exploiting this assumption:
its gradient flow reduces weighted differences between the learned
representations of connected participants, thereby propagating information
over the similarity graph.

The graph is not observed in advance, so its geometry is a modeling decision
rather than a fixed input. Moreover, the same pair of participants may be
highly similar for predicting BFP but less informative for predicting ALM or
BMD. We therefore construct a separate graph for each target by weighting
each anthropometric coordinate according to the magnitude of its correlation
with that target. All correlation weights are estimated exclusively from the
training portion of the current cross-validation fold, preventing held-out
target values from entering graph construction. The resulting target-aware
graph is coupled with the proposed \textbf{State-Adaptive \(p\)-Dirichlet
Energy-Flow Graph Neural Regression} (\(p\)SADE-GNR) model, which performs
nonlinear graph propagation for continuous participant-level regression.

Using data from the Pennington Biomedical Research Center, we compare standard
Euclidean and target-aware correlation-weighted graphs constructed from the
original standardized measurements and from AE, VAE, or GMVAE latent
representations. We evaluate ALM, BFP, and BMD prediction separately in male,
female, and combined-sex cohorts under five-fold cross-validation. We additionally consider biological-age regression\footnote{Throughout this
paper, unless otherwise stated, the term ``age'' refers to biological age
rather than chronological age.} as an exploratory task to examine whether the
same anthropometric measurements capture age-associated biological variation. The correlation-weighted $p$SADE-GNR applied to the original
measurements achieves the lowest RMSE among the evaluated configurations in
all nine primary outcome--cohort comparisons. It also obtains numerically lower
RMSE than the previously reported SVR or LSSVR benchmark in eight of the nine
comparisons, with the largest improvements for ALM and BFP; BMD performance is
more comparable with the benchmark. The AE, VAE, and GMVAE representations do
not improve the primary prediction results and are therefore treated as
representation-learning ablations rather than as central components of the
proposed method.

\subsection*{Summary of Contributions}

The main contributions of this work are summarized as follows:

\begin{enumerate}
   \item \textbf{Target-aware, state-adaptive \(p\)-Dirichlet energy-flow
graph neural regression.}
We formulate a participant-level graph neural regression framework based on the p-Dirichlet energy flow for estimating
Appendicular Lean Mass (ALM), Bone Mineral Density (BMD), and Body Fat
Percentage (BFP) from non-invasive anthropometric measurements. For each
outcome, anthropometric variables are weighted by the magnitudes of their
training-fold correlations with that outcome, producing a leakage-controlled,
outcome-specific participant-similarity graph. During propagation, these fixed
target-aware graph weights are modulated by a normalized \(p\)-dependent factor
derived from the current hidden-state differences, yielding state-adaptive
effective edge weights. The resulting framework separates outcome-specific
graph geometry from nonlinear, state-dependent information propagation without
using held-out target values.

    \item \textbf{Comparison of raw and learned graph representations.}
    We compare participant graphs constructed from standardized anthropometric measurements with graphs constructed from AE, VAE, and GMVAE representations under a common evaluation protocol. For the primary ALM, BMD, and BFP prediction tasks, the target-aware raw-feature state-adaptive \(p\)-Dirichlet Energy Flow model ($p$DEF) provides the strongest overall performance. In contrast, for the exploratory age-prediction task including ALM, BMD, and BFP as predictors, the GMVAE-cc-\(L^2\)--\(p\)SADE-GNR achieves the lowest mean RMSE in the male, female, and combined cohorts. This result shows that learned mixture representations are not uniformly beneficial, but can provide an advantage when modeling age-related variation from combined anthropometric and body-composition information. UMAP is additionally used as an exploratory visualization of target-associated organization in the original and learned feature spaces.

    \item \textbf{Evaluation across multiple outcomes and population groups.}
    We evaluate the proposed models for ALM, BMD, and BFP prediction in male,
    female, and combined-sex cohorts, with Biological Age prediction considered as an
    additional exploratory task. The experiments show that the
    correlation-weighted $p$DEF constructed from the original anthropometric
    measurements provides the strongest overall performance, with especially
    substantial improvements for ALM and BFP compared to the previously
    best reported benchmarks in \cite{gyaneshwar_et_al}. These results also
    demonstrate that increasing the complexity of the latent representation
    does not necessarily improve graph-based regression.
\end{enumerate}

\section{Related Work}\label{sec:related_work}

\subsection{Closely Related Body-Composition Prediction}

Machine-learning models have estimated body composition from both compact
measurements and digital body shape. Among compact-input approaches, Cichosz
et al. predicted total lean and fat mass from anthropometric and demographic
variables, while Birk et al. developed machine-learning equations for several
DXA-derived outcomes from anthropometry and bioimpedance
\cite{birk2025equations,cichosz2021precise}.

Research on digital anthropometry has moved from estimating body
composition with 3D surface scans and engineered shape descriptors
\cite{harty2020bodyfat,ng2016clinical,ng2019shapeup,pleuss2019machinelearning} to models based on conventional photographs, smartphone
images, and device-agnostic 3D representations
\cite{majmudar2022smartphone,mccarthy2023smartphone,tian2020photography,tian2022deviceagnostic}. This literature has also
examined the stability of 3D optical estimates across age, body-mass index, and
ethnicity and their ability to monitor longitudinal change
\cite{wong2023monitoring,wong2023accuracy}. More recent work has developed
outcome-specific neural models, whole-body deep regression, generative
representations, and voxel-based predictors
\cite{bennett2024trunk,feng2026voxel,feng2025enhanced,leong2024generative,marazzato2024alm,qiao2024prediction,tian2025cnn}. These studies establish
that external body shape contains substantial predictive information, but
their imaging inputs, outcomes, cohorts, and validation protocols differ from
our compact anthropometric setting; their reported errors are therefore not
treated as directly comparable benchmarks.

The closest empirical predecessor is Agrahari et al.~\cite{gyaneshwar_et_al}.
That study uses Pennington anthropometric biomarkers
to predict the same three primary outcomes---ALM, BFP, and BMD---and compares
standard supervised regressors, including SVR and LSSVR, with a
game-theoretic $p$-Laplacian method for a limited-label semi-supervised
setting. Its similarity construction assigns equal weight to all biomarkers
and identifies more principled weighting as a direction for further work. Our
study addresses this specific gap by constructing a distinct, training-only
correlation-weighted metric for each target and using it in a supervised $p$SADE-GNR.
Accordingly, the published SVR/LSSVR results provide the most relevant
external benchmarks for our experiments.

In addition, anthropometric measurements have also been used to characterize age-related variation. Tian et al.~\cite{age_pred} analyzed 171 three-dimensional anthropometric measurements from 302 participants in the Baltimore Longitudinal Study of Aging, whose mean age was $(71.7\pm13.4)$ years. Their LASSO model predicted chronological age with a mean absolute error of (4.5) years, and the resulting anthropometric age gap was associated with physical function and lower-extremity body mass. In comparison, our cohort is larger and substantially younger $((n=515), (27.6\pm19.9) years)$, and model performance is reported using root mean squared error rather than mean absolute error. Because the cohorts, prediction settings, and evaluation metrics differ, the numerical errors are not directly comparable; nevertheless, Tian et al. provide an important reference for anthropometry-based age prediction and motivate further development of age-related prediction models.

\subsection{Participant-Similarity and \texorpdfstring{$p$}{p}-Laplacian Learning}

Patient-similarity networks and biomedical population graphs represent
individuals as nodes and encode phenotypic, clinical, or imaging relationships
through edges \cite{kazi2019inceptiongcn,pai2018patient,parisot2018disease}.
These studies motivate relational modeling, but they also highlight a key
challenge when edges are not naturally observed: predictive performance can
depend strongly on how similarity is defined. In our problem, graph
construction is therefore part of the statistical formulation, and a single
fixed metric need not be appropriate for all body-composition targets.

Graph-based learning commonly imposes smoothness over a similarity graph
\cite{belkin2006manifold,zhou2004localglobal,zhu2003gaussian}. Graph
$p$-Laplacians extend the quadratic $p=2$ energy to nonlinear $p$-energies and
have been studied in image and data processing, clustering, semi-supervised
learning, and game-theoretic learning
\cite{buhler2009graphplaplacian,calder2019gameplaplacian,
CalderDrenska,elalaoui2016lp,elmoataz2015graphplaplacian,slepcev2019plaplacian}. Fu et al. incorporated this nonlinear
diffusion principle into a trainable graph neural network \cite{pgnn}. In contrast to Fu et al's work, we developed a distinct diffusion formula directly based on the $p$-Dirichlet energy flow though we have a similar model architecture. The methodological
contribution of the present work is outcome-specific
participant graph and its evaluation for continuous body-composition
regression.

\section{Background}

This section summarizes the two methodological components used in our modeling
framework. We first introduce graph $p$-Dirichlet energy and its gradient flow, which provide the
basis for participant-to-participant information propagation, and then review
deterministic and variational autoencoders, including the Gaussian-mixture
extension, which provide alternative low-dimensional representations for graph
constructions. The
study-specific graph construction, regression architecture, and validation
protocol are described in Section~\ref{sec:methods}.

\subsection{Graph \ensuremath{p}-Dirichlet Energy and Gradient Flow}
\label{sec:pgnn_background}

Neural graph models propagate and aggregate hidden representations along graph
edges \cite{gilmer2017mpnn,kipf2017gcn,scarselli2009gnn,
wu2021gnnsurvey}. In the participant graph considered here, vertices correspond
to samples and weighted edges encode participant similarity. For a continuous
regression target, the relevant notion is \emph{target smoothness}: a graph is
more target-smooth when vertices joined by high-weight edges tend to have
similar response values. We do not assume that every participant graph is
uniformly target-smooth; the degree of local target smoothness may vary across
outcomes, cohorts, and graph-construction procedures.

Graph $p$-energies generalize the quadratic smoothness energy associated with
the ordinary graph Laplacian and have been used in graph-based clustering,
semi-supervised learning, and nonlinear graph diffusion
\cite{buhler2009graphplaplacian,pgnn}. Fu et al.~\cite{pgnn} constructed a
trainable $p$-Laplacian-based graph neural network using adaptive nonlinear
graph filters. The propagation rule used in the present study is related to
this general nonlinear-diffusion principle, but it is not the full
$p$-Laplacian GNN architecture of Fu et al. Instead, it is an explicit
discretization of the gradient flow generated by the graph $p$-Dirichlet
energy.

Let $G=(V,E,W)$ be an undirected weighted graph, where
$W=(w_{ij})$ is symmetric and nonnegative. For a vector-valued graph signal
$u:V\to\mathbb{R}^r$, the graph $p$-Dirichlet energy is
\[
\mathcal{E}_p(u)
=
\frac{1}{p}
\sum_{\{i,j\}\in E}
w_{ij}\lVert u_i-u_j\rVert_2^p,
\qquad p>1.
\]
This energy measures the variation of $u$ across graph edges: it is small when
strongly connected vertices have similar representations.

The negative-gradient flow of $\mathcal{E}_p$ is
\[
\frac{d u_i(s)}{ds}
=
-\nabla_{u_i}\mathcal{E}_p(u(s))
=
-\sum_{j\in\mathcal{N}(i)}
w_{ij}
\lVert u_i(s)-u_j(s)\rVert_2^{p-2}
\bigl(u_i(s)-u_j(s)\bigr),
\]
where the expression is defined as zero when $u_i=u_j$. The gradient
$\nabla\mathcal{E}_p$ is conventionally called the variational graph $p$-Laplacian.
When $p=2$, this is the ordinary linear graph heat flow; when $p\neq2$, the
flow is nonlinear.

Using an explicit Euler discretization with step size $\mu>0$ gives
\[
h_i^{(t+1)}
=
h_i^{(t)}
-
\mu
\sum_{j\in\mathcal{N}(i)}
w_{ij}
\lVert h_i^{(t)}-h_j^{(t)}\rVert_2^{p-2}
\bigl(h_i^{(t)}-h_j^{(t)}\bigr),
\qquad t=0,\ldots,T-1.
\]
Thus, the iterative graph-propagation rule used in this study is a
finite-step discretization of the graph $p$-Dirichlet energy flow.
The model-specific encoder, residual averaging, regression layer, and training
procedure are described in Section~\ref{subsec:pgnn_prediction_model}.

\begin{figure}[!t]
    \centering
    \includegraphics[width=\linewidth]
    {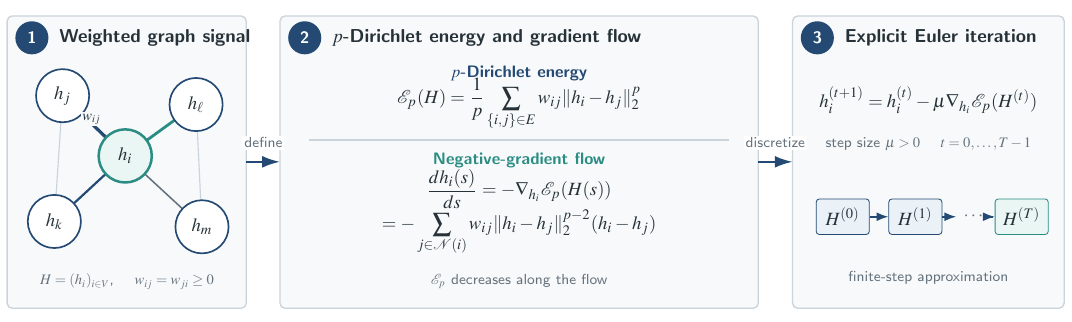}
    \caption{Schematic of the graph $p$-Dirichlet energy-flow module used in
    this study. At each iteration, the representation at vertex $i$ is updated
    by an explicit Euler step in the negative-gradient direction of the graph
    $p$-Dirichlet energy with a fixed step size $\mu$. After $T$ flow steps, the propagated representation
    is averaged with the initial representation and passed to a scalar linear
    regression layer where each $H^{(k)}=[h_i^{(k)}]_{i\in V}$.}
    \label{fig:pgnn_background}
\end{figure}

\subsection{Autoencoders, Variational Autoencoders, and Gaussian-Mixture Extensions} \label{sec:vae_background}

Anthropometric measurements can be strongly correlated and may contain
redundant as well as nonlinear variation. This motivates testing whether a
compact latent representation can facilitate participant-similarity graph
construction. Compression can also remove distinctions that are useful for a
particular prediction target, however, so the value of a latent representation
must be assessed empirically. A deterministic autoencoder learns an
encoder--decoder map through a lower-dimensional bottleneck by minimizing
reconstruction error, providing a standard nonlinear dimensionality-reduction
baseline \cite{hinton2006autoencoder}. We compare this baseline with variational
autoencoders (VAEs) \cite{doersch2021tutorial,kingma2014vae} and
Gaussian-mixture variational autoencoders (GMVAEs)
\cite{dilokthanakul2016gmvae}.

A VAE is a latent-variable generative model that represents an observed sample
$x\in\mathbb{R}^{d}$ using a lower-dimensional latent variable
$z\in\mathbb{R}^{m}$, typically with $m<d$. The model specifies a prior $p(z)$
and a decoder likelihood $p_\theta(x\mid z)$, giving the joint distribution
\[
p_\theta(x,z)=p_\theta(x\mid z)p(z).
\]
Because the posterior $p_\theta(z\mid x)$ is generally intractable, a
neural-network encoder $q_\phi(z\mid x)$ is introduced as a variational
approximation. The VAE is trained by maximizing the evidence lower bound
\[
\log p_\theta(x)
\geq
\mathcal{L}_{\mathrm{VAE}}(x)
=
\mathbb{E}_{q_\phi(z\mid x)}
\!\left[\log p_\theta(x\mid z)\right]
-
D_{\mathrm{KL}}
\!\left(q_\phi(z\mid x)\,\|\,p(z)\right).
\]
The first term of $\mathcal{L}_{\text{VAE}}$ is the expected \emph{reconstruction log-likelihood}: it rewards
the decoder for assigning high probability to the observed measurements $x$
after $z$ is sampled from the encoder distribution. Its negative is often
called the reconstruction loss. For a Gaussian decoder with fixed variance,
maximizing this term is equivalent, up to constants and scaling, to minimizing
a squared reconstruction error. The KL-divergence term regularizes the
approximate posterior toward the prior, promoting a structured latent space
from which samples can be decoded.

A standard VAE commonly uses the unimodal prior
$p(z)=\mathcal{N}(0,I)$. A single Gaussian may be restrictive when the data
distribution contains several modes. A GMVAE introduces a discrete latent
variable $c\in\{1,\ldots,K\}$ and a component-specific Gaussian prior for $z$,
\[
p_\theta(x,z,c)
=
p_\theta(x\mid z)p_\theta(z\mid c)p(c),
\]
with
\[
p_\theta(z\mid c=k)
=
\mathcal{N}(z\mid\mu_k,\Sigma_k),
\qquad k=1,\ldots,K.
\]
This mixture prior allows the latent distribution to represent multiple modes
with different centers and covariance structures rather than forcing all
observations toward one global Gaussian region. The resulting components are
statistical latent components, however, and should not automatically be
interpreted as clinically distinct or biologically verified subgroups without
independent evidence.

Posterior inference in a GMVAE involves both the continuous latent variable $z$
and the component variable $c$. A common variational factorization is
\[
q_\phi(z,c\mid x)
=
q_\phi(c\mid x)q_\phi(z\mid x,c),
\]
where $q_\phi(c=k\mid x)$ is the posterior probability assigned to component
$k$. The vector of these probabilities is called a \emph{soft assignment}
because an observation is not forced into exactly one component; instead, it
may receive nonzero probability under several components, with the
probabilities summing to one. These probabilities can also be used to combine
component-conditional latent means into a single posterior-expected
representation for downstream graph construction.

In this study, the VAE and GMVAE serve as representation-learning modules
rather than final prediction models. The VAE provides a continuous nonlinear
latent representation, while the GMVAE permits a multimodal latent distribution
and probabilistic component membership. Their learned representations are used
to construct alternative participant-similarity graphs for a downstream
$p$-Dirichlet energy flow graph neural model and are compared directly with graphs built from
the standardized anthropometric variables.

\section{Methodology}\label{sec:methods}

This section develops the proposed \textbf{State-Adaptive
\(p\)-Dirichlet Energy-Flow Graph Neural Regression} (\(p\)SADE-GNR)
framework for estimating ALM, BMD, and BFP from anthropometric measurements.
We adopt the graph-smoothness assumption that participants who are similar in
outcome-relevant characteristics should have similar response-related
representations and, consequently, similar predictions. On a weighted
participant graph, the \(p\)-Dirichlet energy measures violations of this
assumption by penalizing differences between connected hidden states in
proportion to their edge weights \cite{elmoataz2008nonlocal}. Reducing this energy therefore provides a
natural mechanism for propagating information among similar participants.
Rather than adding the energy directly to the regression loss, we use its
gradient flow to define graph propagation and introduce a state-adaptive
discretization whose effective edge weights vary with \(p\) and the current
hidden-state differences. Conditions guaranteeing energy decrease are
established below. We evaluate the method using original, AE, VAE, and GMVAE
representations with both unweighted and target-aware correlation-weighted
graphs; age is considered separately as an exploratory target.

\subsection{Study Variables and Prediction Tasks}
\label{subsec:data_tasks}

Let
\[
x_i\in\mathbb{R}^d
\]
denote the accessible input-feature vector for participant \(i\), consisting
of chronological age together with the non-invasive anthropometric
measurements where $d$ is the total number of these quantities. Let
\[
\tilde{x}_i\in\mathbb{R}^{d-1}
\]
denote the corresponding anthropometric feature vector obtained by
removing chronological age from \(x_i\). These measurements include height, weight, body
circumferences, limb and torso volumes, and regional and total
surface-area measurements. The three DXA-derived body-composition outcomes are denoted by
\[
b_i
=
\left(
y_i^{(\mathrm{ALM})},
y_i^{(\mathrm{BMD})},
y_i^{(\mathrm{BFP})}
\right)^\top.
\]
The complete variable list and corresponding units are provided in
Table~\ref{tab:biomarkers_table}. Separate analyses were performed for the
male, female, and combined-sex cohorts.

ALM, BMD, and BFP were treated as three separate single-target regression
tasks. Thus, a distinct target-specific participant graph and prediction model
were fitted for each
\[
q\in\{\mathrm{ALM},\mathrm{BMD},\mathrm{BFP}\},
\]
and the three outcomes were not predicted jointly through a common
multi-output loss.

\paragraph{Primary body-composition tasks.}
For each primary prediction task, the model used chronological age and
the non-invasive anthropometric measurements contained in \(x_i\).
All three DXA-derived body-composition variables were excluded from the
input, including the two body-composition outcomes that were not the
current prediction target. This design prevents the prediction of one
DXA-derived quantity from relying on other DXA-derived measurements
while allowing chronological age to serve as an accessible demographic
predictor. The same input-feature set was used across the three primary
tasks, permitting a consistent comparison of graph-construction and
representation-learning methods.


\paragraph{Exploratory Age task.} Age prediction was considered separately as a secondary exploratory analysis. Two predictor settings were evaluated. The anthropometry-only setting used \[ u_i^{(\mathrm{Age},-)}=\tilde{x}_i, \] whereas the augmented setting included the three measured body-composition outcomes, \[ u_i^{(\mathrm{Age},+)} = \left( \tilde{x}_i^\top, y_i^{(\mathrm{ALM})}, y_i^{(\mathrm{BMD})}, y_i^{(\mathrm{BFP})} \right)^\top. \] The first setting evaluates whether Age can be inferred from the non-invasive anthropometric measurements alone. The second assesses whether the measured body-composition variables provide additional information for Age prediction. Because the augmented setting uses DXA-derived quantities, it is reported only as an exploratory comparison and is not part of the paper's principal non-invasive body-composition estimation claim.




\subsection{Raw and Learned Feature Representations}\label{subsec:representations}

For a given fold, let $r_i\in\mathbb{R}^{m}$ denote the representation used to
construct the graph and $x_i^{(f)} \in \mathbb{R}^d$ be the standardization of raw data $x_i$. Four representation families are considered.

\paragraph{Standardized anthropometric representation.}
For the raw-feature models, $r_i=x_i^{(f)}$ and $m=d$.

\paragraph{Autoencoder representation.}
A deterministic encoder--decoder model was trained to reconstruct the
standardized anthropometric vector. The encoder output was used as the graph
representation,
\[
 r_i=E_{\phi}^{\mathrm{AE}}(x_i^{(f)}).
\]

\paragraph{VAE representation.}
The VAE encoder produced an approximate posterior
$q_{\phi}(z\mid x_i^{(f)})$. The deterministic posterior mean was used for
neighbor search and graph construction,
\[
 r_i=\mathbb{E}_{q_{\phi}(z\mid x_i^{(f)})}[z].
\]
The VAE was fitted within the training portion of each fold. Its objective can
be written as
\[
 \mathcal{L}_{\mathrm{VAE}}
 =
 \mathcal{L}_{\mathrm{rec}}
 +\beta_{\mathrm{KL}}
 D_{\mathrm{KL}}\!\left(q_{\phi}(z\mid x)\,\|\,p(z)\right),
\]
where $\beta_{\mathrm{KL}}$ denotes the KL regularization weight.

\paragraph{GMVAE representation.}
Let $c\in\{1,\ldots,C\}$ denote a mixture component. If
$q_{\phi}(c\mid x_i^{(f)})$ is the posterior component probability and
$\mu_{ic}$ is the component-conditional latent mean for participant $i$, the
representation used for graph construction was the posterior expectation
\[
 r_i
 =
 \sum_{c=1}^{C}q_{\phi}(c\mid x_i^{(f)})\,\mu_{ic}.
\]
The GMVAE was also fitted within each training split. When a supervised
regression term was included in its objective, that term was evaluated only on
training labels and its coefficient was selected inside the inner validation
procedure. 

\subsection{Participant-Similarity Graph Construction}\label{sec:pgnn}
All participants form the set $V$ of vertices where each participant is treated as vertex in $V$. We consider a complete graph $(V, E)$ over $V$ where $(i, j)\in E$ for all $i, j\in V$. A separate graph weight matrix will be constructed
for every target, cohort, representation family, and cross-validation training
split. Let $r_i=(r_{i1},\ldots,r_{im})$ be the selected representation.

\paragraph{Euclidean graph.}
The unweighted representation-space distance was
\[
 d_{f,\mathrm{E}}(i,j)
 =
 \left(\sum_{\ell=1}^{m}(r_{i\ell}-r_{j\ell})^2\right)^{1/2}.
\]

\paragraph{Target-aware correlation-weighted graph.}
For target $q$, feature or latent coordinate $\ell$, and training split
$\mathcal{T}_f$, define the training-only Pearson correlation
\[
 \rho_{f\ell}^{(q)}
 =
 \operatorname{Corr}_{i\in\mathcal{T}_f}
 \!\left(r_{i\ell},y_i^{(q)}\right).
\]
A nonnegative normalized importance weight was then formed as
\[
 a_{f\ell}^{(q)}
 =
 \frac{|\rho_{f\ell}^{(q)}|}
 {\sum_{s=1}^{m}|\rho_{fs}^{(q)}|+\varepsilon},
\]
and the correlation-weighted distance was
\[
 d_{f,\mathrm{cc}}^{(q)}(i,j)
 =
 \left[
 \sum_{\ell=1}^{m}
 a_{f\ell}^{(q)}(r_{i\ell}-r_{j\ell})^2
 \right]^{1/2}.
\]
For VAE and GMVAE graphs, the correlations were calculated between the latent
coordinates and the target, rather than reusing correlations from the original
anthropometric variables. Correlations calculated from the full dataset were
used only for descriptive figures such as the correlation clock and were not
used in held-out prediction.


\paragraph{Mutual $k$-nearest-neighbor topology and edge weights.}
For a chosen distance $d_f$, let $\mathcal{N}_{k,f}(i)$ be the set of the
$k$ nearest eligible training nodes to node $i$. The mutual-neighbor training
edge set was
\[
 (i,j)\in E_f
 \quad\Longleftrightarrow\quad
 j\in\mathcal{N}_{k,f}(i)
 \ \text{and}\
 i\in\mathcal{N}_{k,f}(j).
\]
Edge weights were assigned using the exponential kernel
\[
 w_{ij}^{(f)}
 =
 \exp\!\left[-\alpha_f d_f(i,j)\right],
 \qquad (i,j)\in E_f,
\]
where $\alpha_f>0$ was determined from the mean edge distance scale of the corresponding graph. The mutual $k$ nearest neighbor
adjacency matrix was symmetric and had zero diagonal before any self-loop or
normalization operation used by the $p$SADE-GNR implementation.

\paragraph{Theoretical motivation.}
Let \(R=(R_1,\ldots,R_m)^\top\) be a standardized participant profile,
\(\rho_\ell=\operatorname{Corr}(R_\ell,Y)\), and suppose that

$$
\mathbb{E}[Y\mid R=r]=\beta_0+\beta^\top r.
$$

For feature weights \(b\in\Delta_m=\{b_\ell\geq0:
\sum_{\ell=1}^m b_\ell=1\}\), define

\begin{equation}\label{eq:linear_expected_target}
d_b(r,r')
=
\left[\sum_{\ell=1}^m b_\ell(r_\ell-r_\ell')^2\right]^{1/2},
\qquad
K(b)
=
\sup_{\delta\ne0}
\frac{|\beta^\top\delta|}
{\left(\sum_{\ell=1}^m b_\ell\delta_\ell^2\right)^{1/2}},
\end{equation}

where \(K(b)=\infty\) if \(b_\ell=0\) for some \(\beta_\ell\ne0\).

\begin{proposition}[Correlation weighting in an idealized linear model]
\label{prop:target_aware_linear_guarantee}
Suppose \(\beta\ne0\) and
\(|\beta_\ell|=c|\rho_\ell|\) for some \(c>0\). Then

$$
K(b)^2
=
\sum_{\ell:\beta_\ell\ne0}\frac{\beta_\ell^2}{b_\ell},
$$

and its unique minimizer over \(\Delta_m\) is

$$
b_\ell^\star
=
\frac{|\beta_\ell|}{\|\beta\|_1}
=
\frac{|\rho_\ell|}{\|\rho\|_1},
\qquad
K(b^\star)=\|\beta\|_1.
$$

For equal weighting \(b_\ell^{\mathrm{Euc}}=1/m\),

\begin{equation}\label{eq:target_aware_not_worse}
    \frac{K(b^\star)}{K(b^{\mathrm{Euc}})}
    =
    \frac{\|\rho\|_1}{\sqrt m\,\|\rho\|_2}
    \leq1,
\end{equation}

with equality if and only if all features have equal absolute correlation
with the target.
\end{proposition}

Because

$$
|\,\mathbb{E}[Y\mid R=r]-\mathbb{E}[Y\mid R=r']\,|
\leq K(b)d_b(r,r'),
$$

In this simplified linear setting, Proposition~\ref{prop:target_aware_linear_guarantee}
shows that weighting each measurement by its absolute correlation with the
target gives the smallest worst-case bound on the expected-outcome difference
between nearby participants. The required condition holds automatically when
the standardized measurements are uncorrelated. In real clinical data,
however, measurements may be correlated, and the correlation of one
measurement with the target may not represent its contribution after the
other measurements are considered. The proposition therefore explains the
motivation for correlation weighting but does not claim that it always
produces the optimal graph or improves prediction accuracy in finite samples.
The proof is given in
Appendix~\ref{app:target_aware_linear_guarantee}.

\subsection{State-Adaptive \ensuremath{p}-Dirichlet Energy-Flow Graph Neural Regression(\ensuremath{p}SADE-GNR)}
\label{subsec:pgnn_prediction_model}

Let $G=(V,E,W)=(w_{ij})_{\{i, j\}\in E})$ be a finite undirected graph with fixed symmetric weights $w_{ij}>0$ on $E$. In practice, $G$ is constructed out of given data following the method described above \ref{sec:pgnn}.
The representation $r_i$ of each raw data $x_i$ attached at node $i$ was mapped to an initial hidden state by a multi-layer perceptron
\[
h_i^{(0)}=\text{MLP}_\theta(r_i),
\]
where $\text{MLP}_\theta$ contains two fully connected layers with 128 units per layer,
batch normalization, ReLU activation, and dropout $0.20$.

In general, a hidden-state matrix is $H=[h_i]_{i\in V}\in\mathbb{R}^{|V| \times r}$ such that each $h_i\in\mathbb{R}^r$ for $i\in V$ be the hidden state at vertex $i$ and define
\[
\delta = \max_{(i, j)\in E}\left\|h_i-h_j\right\|_2
\]
In our scenario, $H^t=[h_i^t]_{i\in V}\in\mathbb{R}^{|V| \times r}$ is the hidden-state matrix after iteration $t$ and $\delta_t$ is defined as above for this $H^t$.

When $\delta_t>0$, we use the adaptive effective step size
\[
\tau_t^{(p)}
=
\mu\left(\delta_t\right)^{2-p},
\]
and then update
\begin{align}
h_i^{(t+1)}
&=
h_i^{(t)}
-
\tau_t^{(p)}
\sum_{j\in\mathcal N(i)}
w_{ij}
\left\|h_i^{(t)}-h_j^{(t)}\right\|_2^{p-2}
\left(h_i^{(t)}-h_j^{(t)}\right) \notag\\
&=
h_i^{(t)}
-
\mu
\sum_{j\in\mathcal N(i)}
w_{ij}
\left(
\frac{\left\|h_i^{(t)}-h_j^{(t)}\right\|_2}
     {\delta_t}
\right)^{p-2}
\left(h_i^{(t)}-h_j^{(t)}\right),
\label{eq:adaptive_pdirichlet_update}
\end{align}
for $t=0,\ldots,T-1$. 

If $\delta_t=0$, the hidden state is already
constant along every graph edge and no further propagation is required.

Equation~\eqref{eq:adaptive_pdirichlet_update} is an adaptive forward-Euler
discretization of the graph $p$-Dirichlet energy flow. The normalization keeps
the $p$-dependent edge factors in $[0,1]$ and is central to the stability of
the propagation for large $p$, as established below.

Our predictive architecture
additionally introduces the residual average
\[
\widetilde h_i
=
\frac12 h_i^{(T)}
+
\frac12 h_i^{(0)}.
\]
This averaging is not part of the underlying $p$-Dirichlet energy flow; it is a
model-specific residual connection that retains information from the initial
encoded representation. A linear output layer then gives
\begin{equation}
\widehat y_i
=
W_{\mathrm{out}}\widetilde h_i+b_{\mathrm{out}}.
\label{pred_head}
\end{equation}

Only labels in the current training partition contribute to the mean-square error:
\[
\mathcal L_{mse}(\Theta)
=
\frac{1}{|\mathcal T_f|}
\sqrt{\sum_{i\in\mathcal T_f}
\left(y_i^{(q)}-\widehat y_i^{(q)}\right)^2}.
\]
For $p=2$, $\tau_t^{(2)}=\mu$ and the propagation reduces to the standard
weighted graph heat-flow update. For $p>2$, the adaptive normalization controls
the nonlinear edge contributions and permits stable propagation at much larger
values of $p$.

To establish that the proposed state-adaptive update retains the energy-dissipating behavior of the \(p\)-Dirichlet flow, the following theorem provides conditions under which the weighted \(p\)-Dirichlet energy,

\begin{equation}
\boxed{
\E_p(H)
=
\frac{1}{p}
\sum_{\{i,j\}\in E}
w_{ij}\lVert h_i-h_j\rVert_2^p
}
\label{eq:objective}
\end{equation}
is nonincreasing at every propagation step and strictly decreases away from equilibrium. The proof can be found in \ref{proof:shrinking}

\begin{theorem}[Shrinking property of state-adaptive $p$-Dirichlet flow]
\label{thm:shrinking}
Assume $p\ge 2$ and Let $G=(V, E, W)$ be a weighted graph as above. By convention, the weighted graph Laplacian $L_w = D_w - W$ where $D_w=\text{diag}(d_1, d_2, \dotso, d_n)$ where each $d_i=\sum_{j=1}^nw_{ij}$. Let $d_w=\max_i\{d_i\}$ and $\Lambda_w=\lambda_{\max}(L_w)$ which is the largest eigenvalue of the weight graph Laplacian.
Define
\begin{equation}
R_\mu=1+2\mu d_w,
\qquad
\chi_{p,\mu}
=
\mu(p-1)\Lambda_w R_\mu^{p-2}.
\label{eq:chi}
\end{equation}
If $\chi_{p,\mu}\le 2$, then for every propagation step $t \geq 0$, the update \eqref{eq:adaptive_pdirichlet_update} satisfies
\begin{align}
\E_p(H^{(t+1)})
&\le
\E_p(H^{(t)}) \notag\\
&\quad-
\mu\left(\delta_t\right)^{2-p}
\left(1-\frac{\chi_{p,\mu}}{2}\right)
\lVert\nabla\E_p(H^{(t)})\rVert_F^2.
\label{eq:quantitative-descent}
\end{align}
If $\chi_{p,\mu}<2$, the decrease is strict whenever $H^{(t)}$ is not
constant on every positive-weight connected component.  If $\delta_t=0$,
then $\E_p(H^{(t)})=0$ and the prescribed stopping rule gives equality.

Therefore, if $\chi_{p,\mu}<2$, every propagation step shrinks the weigheted p-Dirichlet energy 
$$\E_p(H^{(t+1)})\le\E_p(H^{(t)}).$$
\end{theorem}
\begin{remark}
    In our experiments, the step size $\mu$ is sufficiently small such that $p\mu\ll 1$, which means that $\mu$ scales approximately as $1/p$ or even smaller ($\mu\leq 10^{-4}$ is always the case when $p$ is small). Therefore, most of time, $\chi_{p, \mu}\le 2$, which, by Theorem \ref{thm:shrinking}, indicates that the $p$-Dirichlet energy will decrease during the iterative updating process. 
\end{remark}

\begin{algorithm}[h]
\caption{Training the \(p\)SADE-GNR model}
\label{alg:psade_gnr}
\begin{algorithmic}

\REQUIRE Fold-specific graph \(G=(V,E,W)\), node representations
\(\{r_i\}_{i\in V}\), training set \(\mathcal T_f\), targets
\(\{y_i^{(q)}\}_{i\in\mathcal T_f}\), \(p\geq2\), \(\mu>0\),
propagation depth \(T\), and training epochs \(N_{\mathrm{ep}}\)

\ENSURE Trained parameters
\(\Theta=\{\theta,W_{\mathrm{out}},b_{\mathrm{out}}\}\)

\STATE Initialize \(\Theta\)

\FOR{\(e=1,\ldots,N_{\mathrm{ep}}\)}

    \STATE Encode
    \(h_i^{(0)}\leftarrow\mathrm{MLP}_{\theta}(r_i)\)
    for all \(i\in V\)

    \FOR{\(t=0,\ldots,T-1\)}

        \STATE Compute
        \(\displaystyle
        \delta_t\leftarrow
        \max_{\{i,j\}\in E}
        \|h_i^{(t)}-h_j^{(t)}\|_2
        \)

        \IF{\(\delta_t=0\)}
            \STATE Set \(h_i^{(T)}\leftarrow h_i^{(t)}\)
            for all \(i\in V\), and terminate propagation
        \ELSE
            \STATE Simultaneously update all hidden states using
            \eqref{eq:adaptive_pdirichlet_update}
        \ENDIF

    \ENDFOR

    \STATE Form the residual representation
    \(\displaystyle
    \widetilde h_i
    \leftarrow
    \tfrac12 h_i^{(T)}+\tfrac12h_i^{(0)}
    \)
    for all \(i\in V\)

    \STATE Predict
    \(\displaystyle
    \widehat y_i^{(q)}
    \leftarrow
    W_{\mathrm{out}}\widetilde h_i+b_{\mathrm{out}}
    \)
    for all \(i\in V\)

    \STATE Compute
    \(\displaystyle
    \mathcal L_{\mathrm{mse}}
    \leftarrow
    \frac{1}{|\mathcal T_f|}
    \sum_{i\in\mathcal T_f}
    \bigl(y_i^{(q)}-\widehat y_i^{(q)}\bigr)^2
    \)

    \STATE Update \(\Theta\) by backpropagation

\ENDFOR

\RETURN \(\Theta\)

\end{algorithmic}
\end{algorithm}

\subsection{Hyperparameter Selection and Implementation}\label{subsec:hyperparameters}

The neighborhood size was selected from
\[
 k_{\mathrm{NN}}
 \in
 \{2,5,10,15,20,25,30,35,40\}.
\]
The graph $p$-Laplacian exponent was selected from
\[
 p\in\{2,3,5,10,10^{2},10^{4},10^{6}\}.
\]
For each $p$, the candidate numbers of diffusion iterations were \label{p_range}
\[
\begin{aligned}
 p=2      &: \quad T\in\{5,10\},\\
 p=3      &: \quad T\in\{10,15\},\\
 p=5      &: \quad T\in\{15,30\},\\
 p=10     &: \quad T\in\{40,80\},\\
 p=10^{2} &: \quad T\in\{100,200\},\\
 p=10^{4} &: \quad T\in\{400,800\},\\
 p=10^{6} &: \quad T\in\{1600,3200\}.
\end{aligned}
\]
All candidate configurations were compared only within the inner validation
procedure. The diffusion step size $\mu$ was computed as a deterministic
function of $p$ and $T$ and was reduced by a fixed multiplicative back-off
factor if the diffusion produced non-finite activations or loss values. The
selected hyperparameters were then refitted using the complete outer training
set before evaluation on the held-out fold.


The $p$SADE-GNR was optimized using AdamW and a one-cycle learning-rate schedule over
$[10^{-4},10^{-2}]$. Training was limited to 600 epochs, with early stopping
after 30 epochs without improvement in the inner-validation loss. Experiments
used random seed 42 and the same CUDA-enabled computing environment across
model configurations. 

The source code is publicly available on GitHub at https://github.com/GowriPriyaSunkara/pLaplacian-Graph-NN and https://github.com/jonathanwang0514/pLaplacian-Graph-NN.

\subsection{Model Configurations}\label{para:summary}

Seven representation--graph configurations were evaluated:
\begin{itemize}
    \item \textbf{$L^2$--$p$SADE-GNR}: standardized anthropometric representation with
    Euclidean graph distance;

    \item \textbf{cc-$L^2$--$p$SADE-GNR}: standardized anthropometric representation with target-aware correlation -weighted distance;

    \item \textbf{VAE-$L^2$--$p$SADE-GNR}: VAE representation with Euclidean distance;

    \item \textbf{VAE-cc-$L^2$--$p$SADE-GNR}: VAE representation with latent-coordinate
    correlation weighting;

    \item \textbf{GMVAE-$L^2$--$p$SADE-GNR}: GMVAE expected representation with
    Euclidean distance;

    \item \textbf{GMVAE-cc-$L^2$--$p$SADE-GNR}: GMVAE expected representation with
    latent-coordinate correlation weighting; and

    \item \textbf{AE-$L^2$--$p$SADE-GNR}: deterministic autoencoder representation with
    Euclidean distance.
\end{itemize}
This comparison separates the effects of representation learning and
outcome-specific graph construction. In particular, comparison of $L^2$--$p$SADE-GNR
with cc-$L^2$--$p$SADE-GNR isolates the contribution of target-aware graph geometry,
whereas comparison of $p=2$ with tuned $p$ values assesses whether nonlinear
$p$-Laplacian diffusion contributes beyond a standard graph Laplacian. The flow charts of some of these configurations are depicted in figures \ref{fig:group1} and \ref{fig:vae}.

\begin{figure}[!t]
    \centering
    \includegraphics[width=0.85\textwidth]{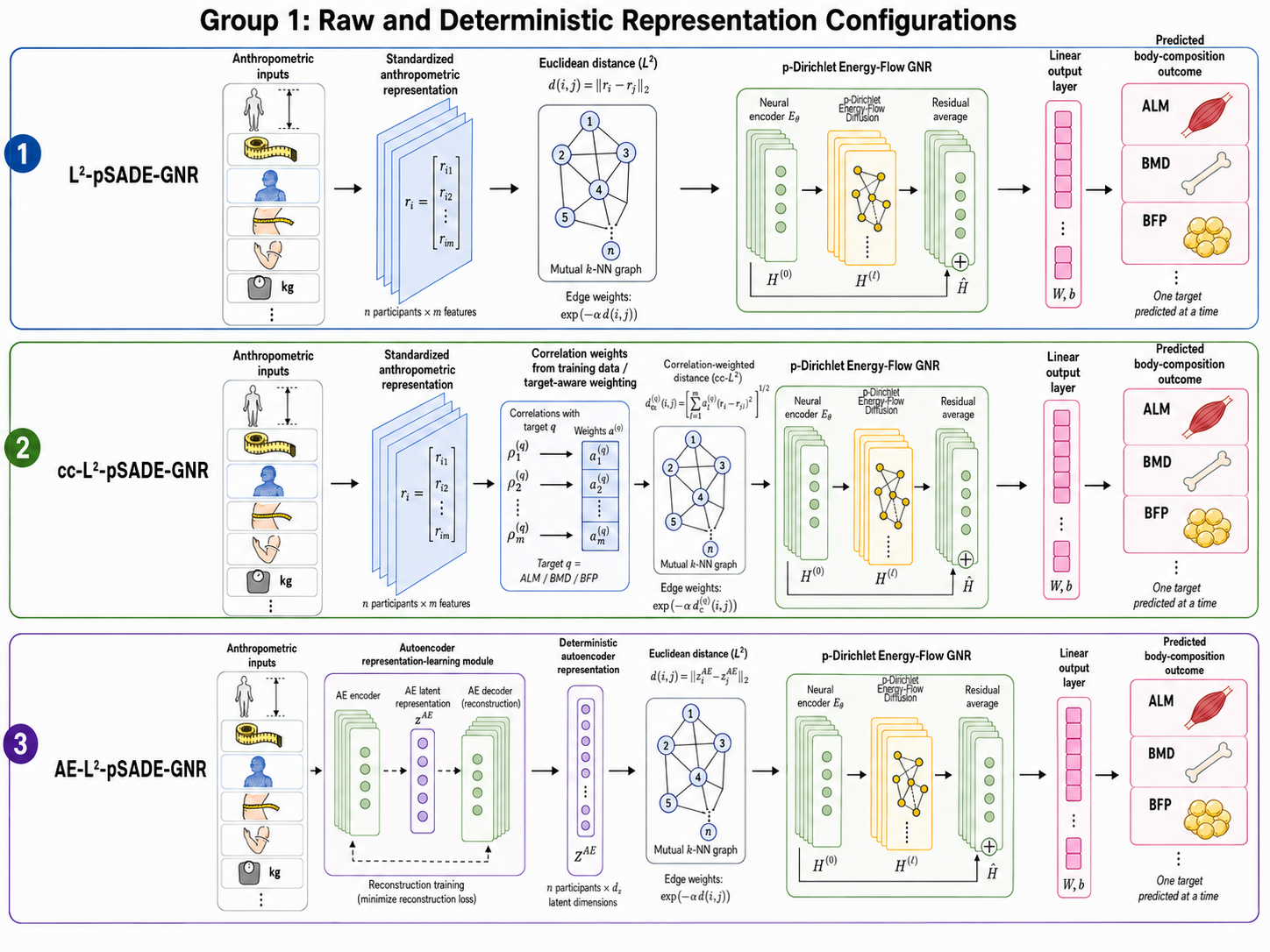}
    \caption{Schematic comparison of three $p$-Dirichlet Energy-Flow Graph Neural Regression ($p$SADE-GNR) pipelines. (1) $L^2$-$p$SADE-GNR: standardized anthropometric measurements are used directly to construct a mutual $k$-nearest-neighbor graph using Euclidean distance. (2) cc-$L^2$–$p$SADE-GNR: each standardized measurement is weighted according to its training-set correlation with the prediction target, producing a target-specific correlation-weighted distance before graph construction. (3) AE-$L^2$-$p$SADE-GNR: a deterministic autoencoder first maps the anthropometric measurements to a lower-dimensional latent representation, from which the graph is constructed using Euclidean distance. In all three configurations, the resulting participant-similarity graph is processed by the same $p$-SADE-GNR architecture. Appendicular lean mass (ALM), bone mineral density (BMD), and body fat percentage (BFP) are predicted as separate single-target regression tasks.}
    \label{fig:group1}
\end{figure}

\begin{figure}[!t]
    \centering
    \includegraphics[width=0.85\textwidth]{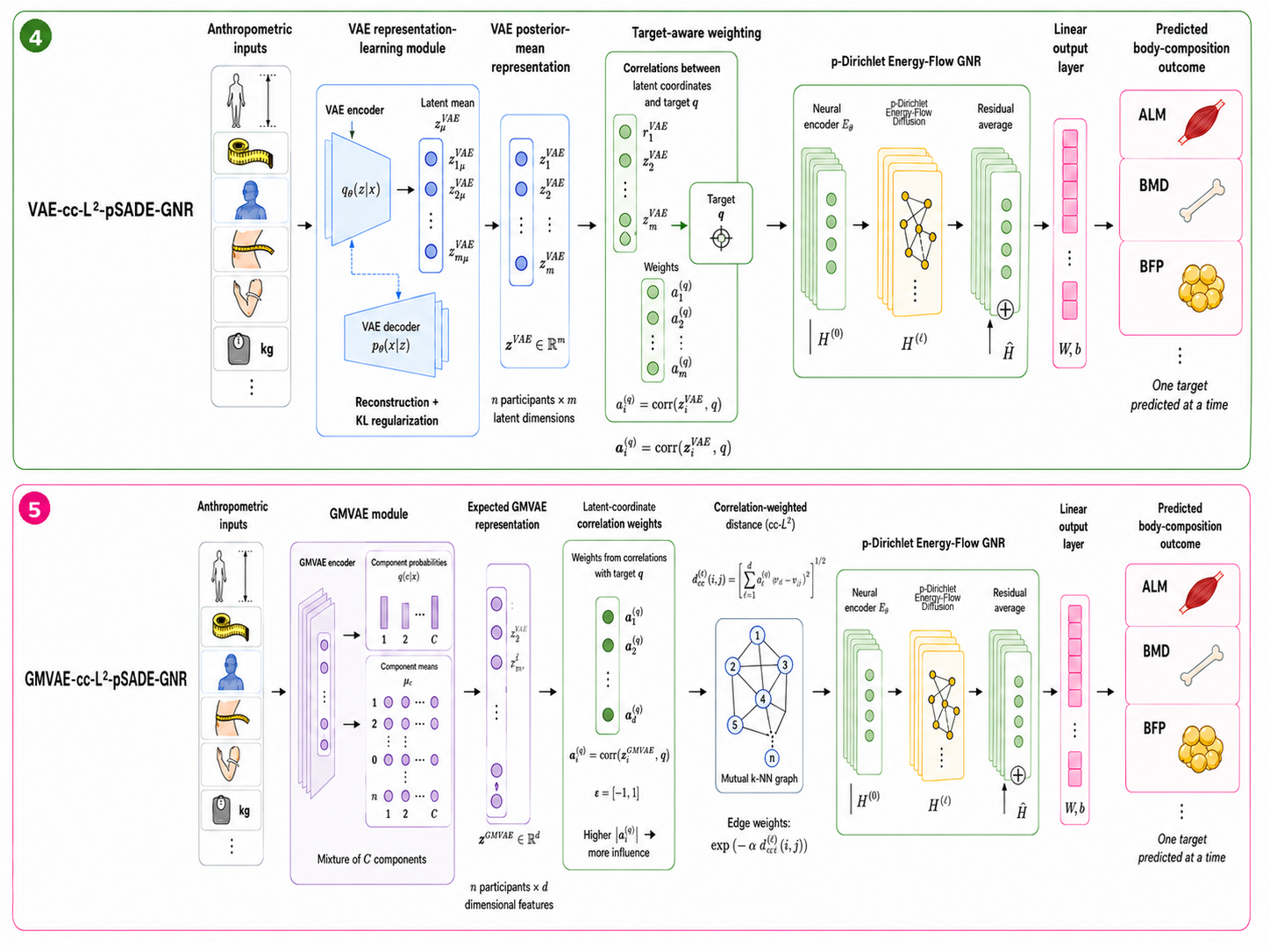}
    \caption{Schematic comparison of the two correlation-weighted variational pipelines. (4) VAE-cc-$L^2$– $p$SADE-GNR: a variational autoencoder maps the standardized anthropometric measurements to a continuous latent space, and the posterior-mean representation is used for graph construction. Latent coordinates are weighted according to their training-set correlations with the current target, yielding a target-aware correlation-weighted distance. (5) GMVAE-cc-$L^2$–$p$SADE-GNR: a Gaussian-mixture variational autoencoder represents each participant by the posterior expectation over mixture-specific latent means, after which latent-coordinate correlations are used to define the target-aware distance. In both configurations, the weighted distance determines a mutual $k$-nearest-neighbor participant graph and its distance-based edge weights. The graph is then processed by the $p$SADE-GNR to predict ALM, BMD, or BFP, with one outcome modeled at a time.}
    \label{fig:vae}
\end{figure}

\section{Experimental Results}\label{sec:results}

The predicted anthropometric measurements are Appendicular Lean Mass (ALM), Bone Mineral Density (BMD), and Body Fat Percentage (BFP). There were seven models used to predict each of these measurements. For each of these measurements, the data was partitioned by sex into male, female, and combined. Outlined in the table below (Table \ref{tab:avgkfold_results_relative_RMSE}) the lowest relative root mean square error percentage ($\mathrm{RMSE}_{\mathrm{relativepct}}$) of each of the models run on each of the anthropometric measurements. Recall from Section \ref{para:summary} that $L^2$ means that the graph constructed for the training of the $p$SADE-GNR uses the standard Euclidean distance, $cc$ means that correlation coefficients were used as weights on the distance function, VAE means that the variational autoencoder was used before the graph was constructed, and GMVAE means that the Gaussian mixture variational autoencoder was used before the construction of the graph. In each of the tables below, the best $\mathrm{RMSE}_{\mathrm{relativepct}}$ obtained for each sex is double-underlined. In Table  \ref{tab:avgkfold_results_relative_RMSE}, the results are compared with the results from \cite{gyaneshwar_et_al}. The models that performed best in their paper were either Support Vector Regression (SVR) or Least Squares Support Vector Regression (LSSVR). In Table \ref{tab:avgkfold_results_relative_RMSE} the models' $\mathrm{RMSE}_{\mathrm{relativepct}}$ are compared with each other to find the best $\mathrm{RMSE}_{\mathrm{relativepct}}$ values obtained for each sex.

\paragraph{Dataset.}The experiments were conducted using a real, non-public participant-level dataset collected at the Pennington Biomedical Research Center (PBRC). The analytic cohort comprised (515) unique participants, including (270) female and (245) male participants, with ages ranging from (5) to (77) years (mean $(27.6\pm19.9) years)$. For each participant, the dataset contains (43) non-invasive anthropometric measurements, including height and weight, body circumferences, limb lengths, regional volumes, and body-surface-area measurements. The prediction targets were three body-composition outcomes obtained from dual-energy X-ray absorptiometry (DXA): appendicular lean mass (ALM, kg), total bone mineral density (BMD, ($\mathrm{gm/cm^2})$), and total body fat percentage (BFP, \%). Separate experiments were conducted for the male, female, and combined-sex cohorts. Qualified researchers may direct reasonable access requests to the final two authors; any data release will be subject to PBRC approval, applicable ethical and privacy requirements, and execution of an appropriate data-use agreement.

\subsection{Exploratory UMAP Visualization}
\label{subsec:umap_results}

Figures~\ref{fig:umap_alm} and~\ref{fig:umap_age} provide exploratory
two-dimensional visualizations of the standardized raw anthropometric
features, the VAE representation, and the GMVAE expected representation.
UMAP was used only for visualization and was not used to construct the
prediction graph or produce the reported predictions.

Figure~\ref{fig:umap_alm} shows the three representations colored by
standardized ALM. A broad ALM-associated gradient is visible in each
embedding. In the raw-feature representation, lower ALM values occur more
frequently toward one end of the embedding, whereas higher values are more
common toward the opposite end. The VAE and GMVAE embeddings retain similar
broad target-associated organization, although substantial local overlap
between ALM values remains in all three representations.

The GMVAE embedding exhibits a more branched geometry than the raw-feature
and VAE embeddings. Here, branched geometry refers to the appearance of multiple elongated regions or arms extending from narrower connecting regions. In particular, an apparent branching region is visible near the central-left portion of the GMVAE embedding, where the point cloud separates into an upper and lower/central branches. Additional curved extensions are visible toward the upper-right and lower-right portions of the embedding. By comparison, the raw-feature UMAP forms a broad, irregular cloud, whereas the VAE embedding displays a pronounced horseshoe-shaped outer band with several smaller groups in its interior. This geometric separation should not, by itself, be
interpreted as evidence of distinct biological subgroups. In addition,
because the GMVAE representation includes a supervised regression component,
the ALM-associated organization in this embedding is not an independent
unsupervised validation of subgroup structure.

Figure~\ref{fig:umap_age} shows the same representations colored by
standardized Age. Compared with ALM, Age exhibits a weaker and more
fragmented regional pattern. Individuals with different ages overlap
substantially throughout the raw-feature, VAE, and GMVAE embeddings, although
some localized age-associated variation is visible. These visualizations
therefore suggest that ALM is more consistently aligned with the local
two-dimensional organization displayed by UMAP, whereas Age remains more
heterogeneously distributed.

These observations are descriptive and should not be interpreted as a
quantitative assessment of representation quality. In particular, the shape,
orientation, and separation of a two-dimensional UMAP embedding do not by
themselves establish predictive superiority, biological clustering, or
preservation of global distances.
\begin{figure}[H]
    \centering
    \includegraphics[width=0.32\textwidth]{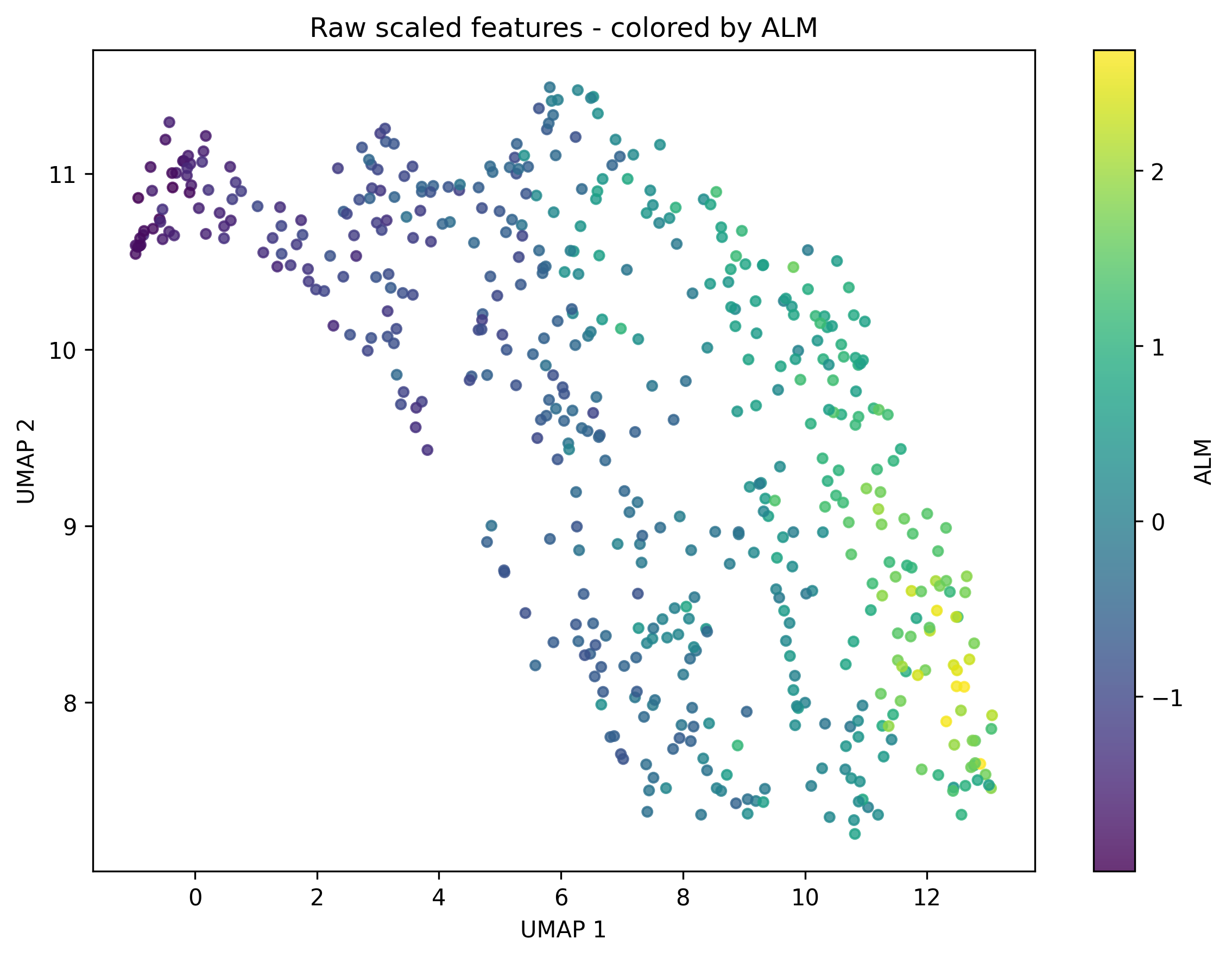}
    \includegraphics[width=0.32\textwidth]{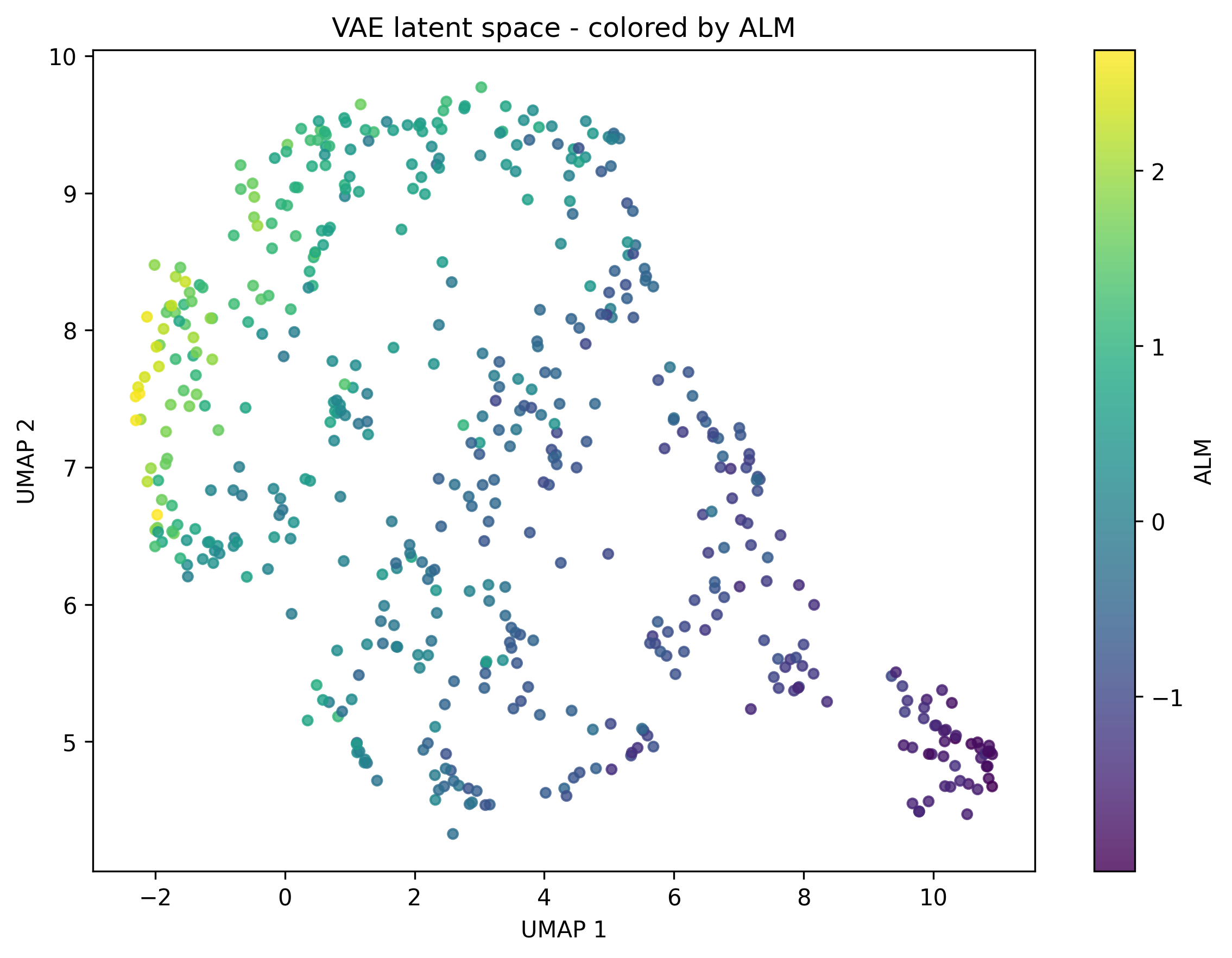}
    \includegraphics[width=0.32\textwidth]{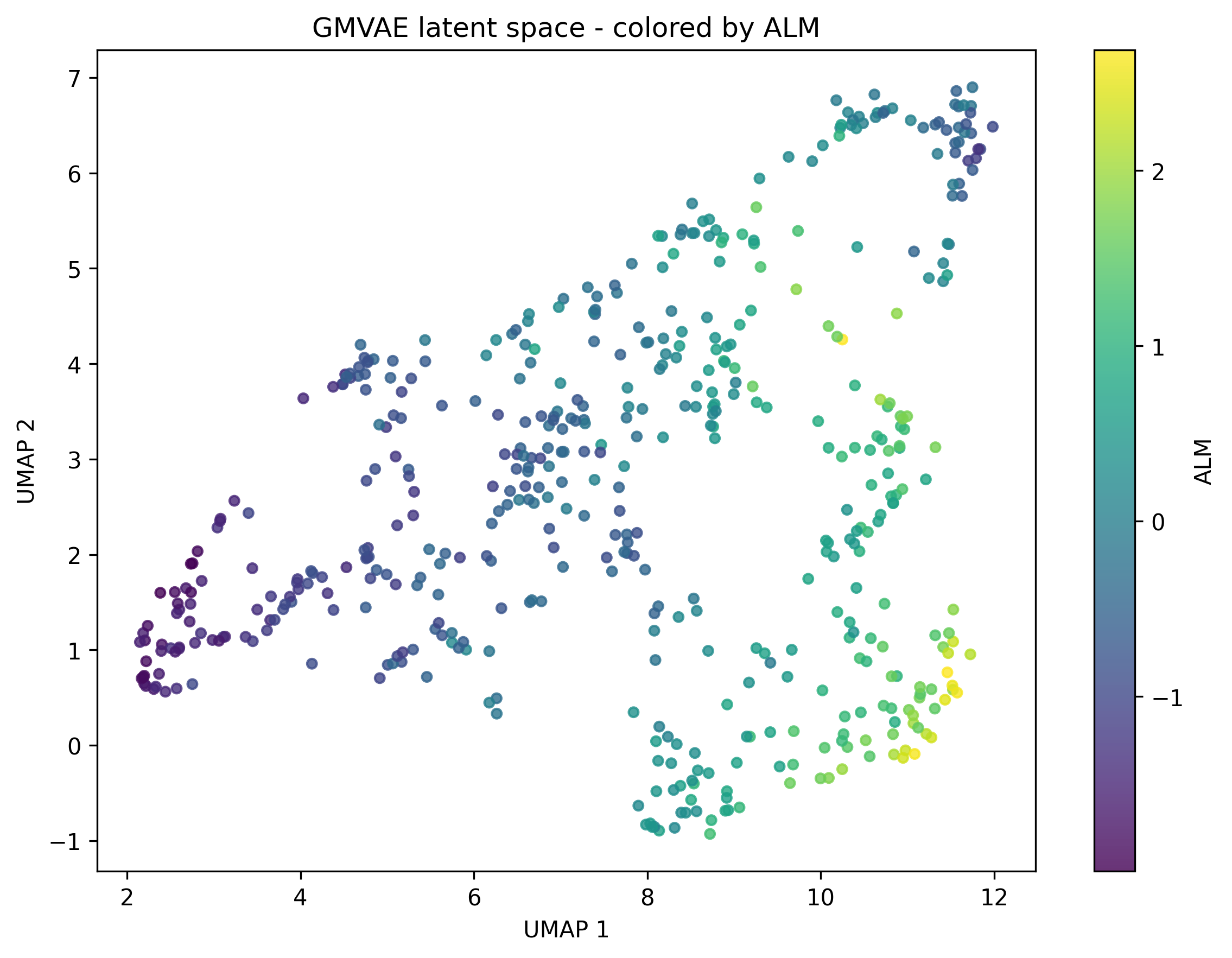}
    \caption{Two-dimensional UMAP embeddings of the Pennington data colored by
standardized ALM. Left: standardized raw anthropometric features.
Middle: VAE latent representation. Right: GMVAE expected latent
representation. The same ALM color scale is used across all three panels.
UMAP was used for exploratory visualization only.}
    \label{fig:umap_alm}
\end{figure}
\begin{figure}[H]
    \centering
    \includegraphics[width=0.32\textwidth]{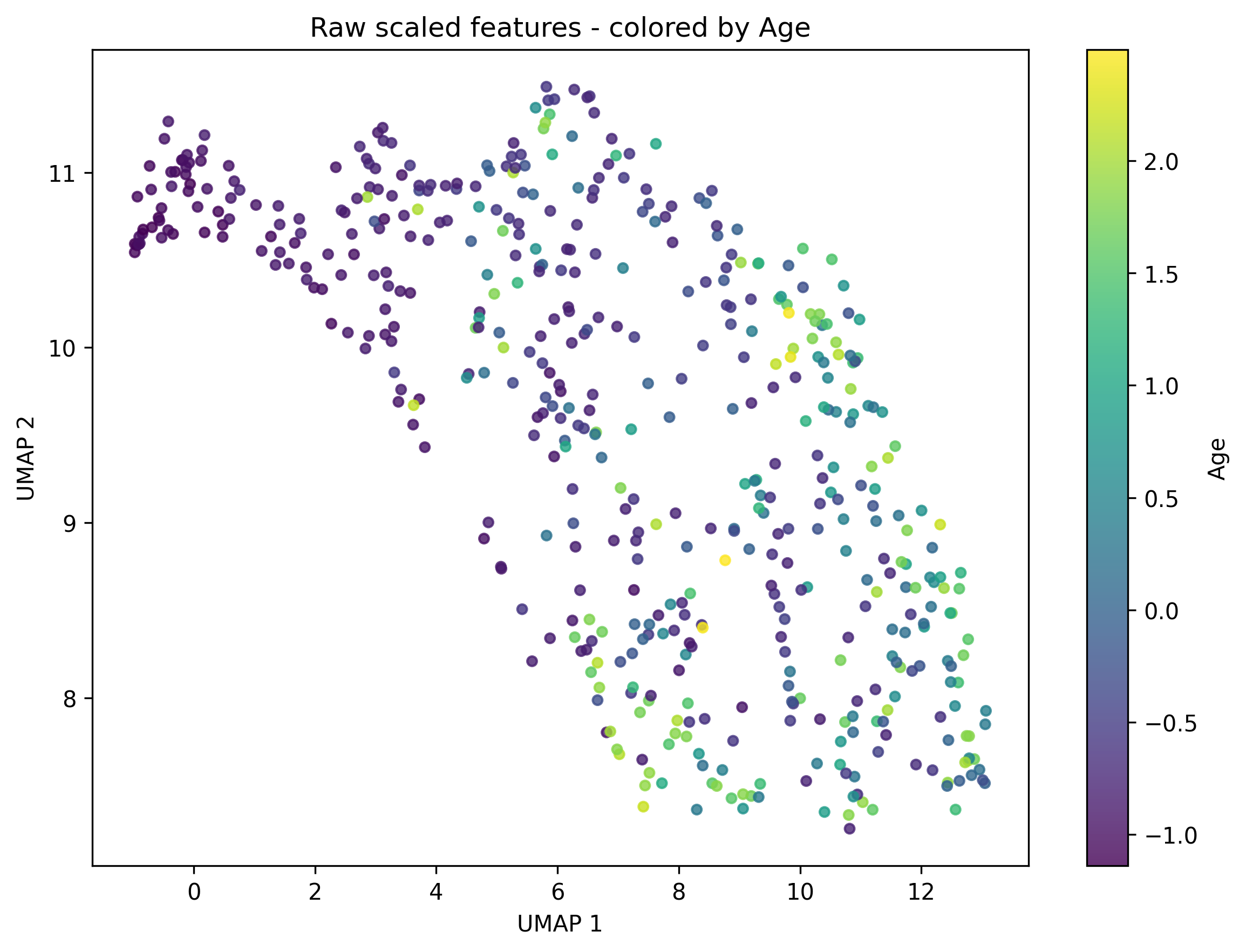}
    \includegraphics[width=0.32\textwidth]{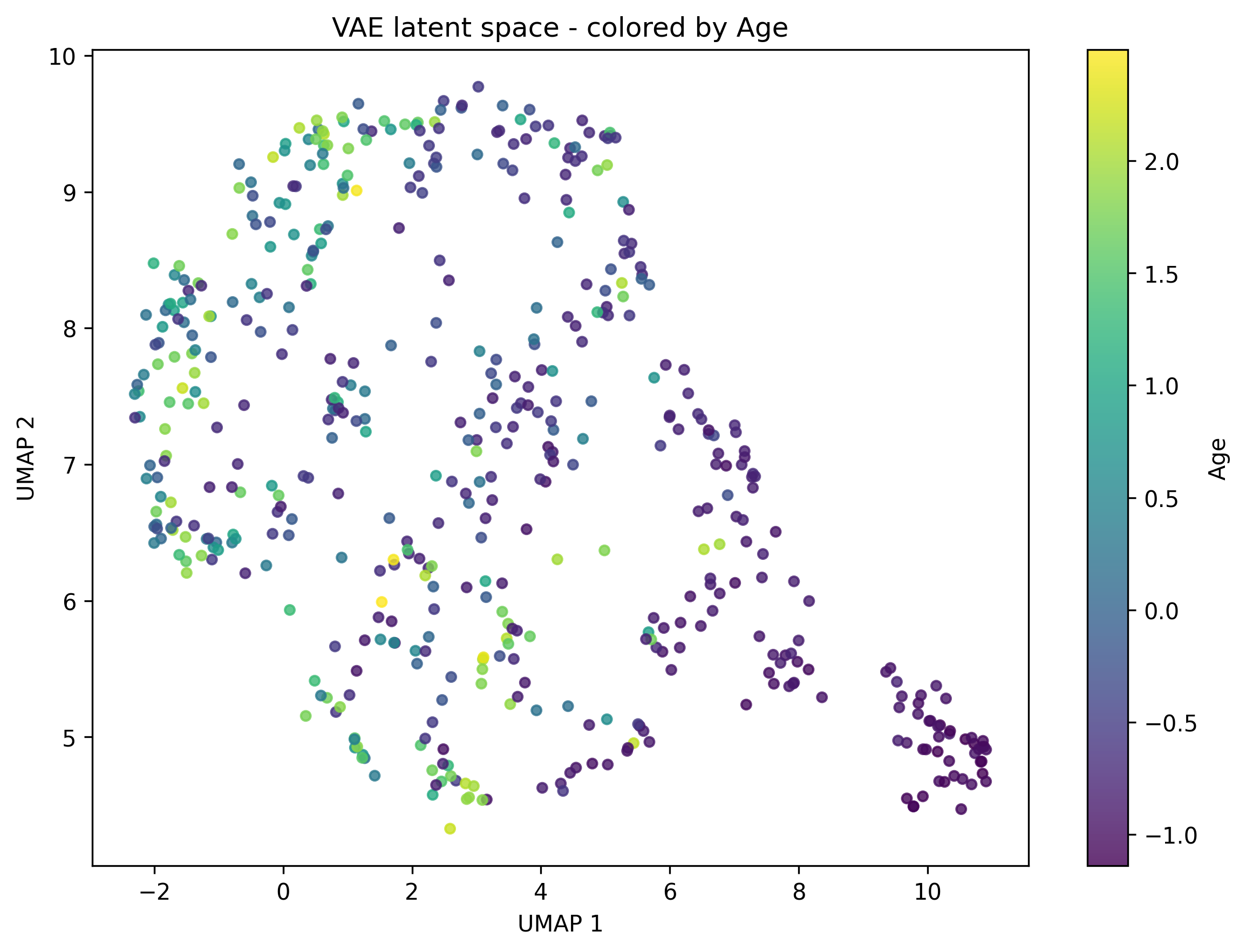}
    \includegraphics[width=0.32\textwidth]{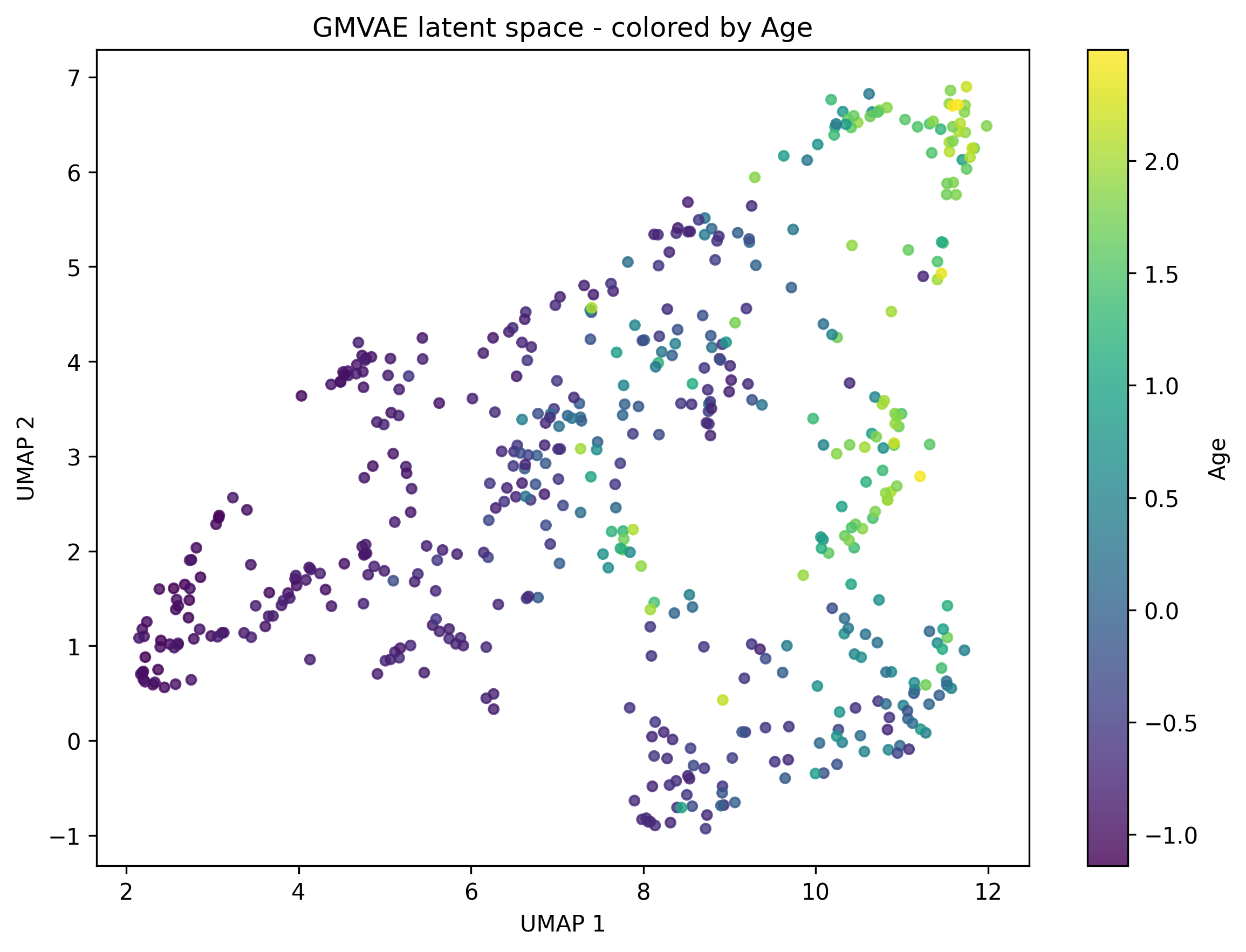}
    \caption{Two-dimensional UMAP embeddings of the Pennington data colored by
standardized Age. Left: standardized raw anthropometric features.
Middle: VAE latent representation. Right: GMVAE expected latent
representation. The same Age color scale is used across all three panels.
UMAP was used for exploratory visualization only. }
    \label{fig:umap_age}
\end{figure}

\subsection{MAPPER Visualization}

Figure \ref{fig:Penn_data_mapper} illustrates the clusters in the dataset as vertices and if there are data points that are shared between clusters, there is a single edge between the vertices. This visualization can provide some insight into some topological features that may be present in the dataset during different embedding methods. Here the Variational Autoencoder (VAE) and the Gaussian Mixture Variational Autoencoder (GMVAE) were used to embed the data in a lower dimension. After each of these methods, the MAPPER algorithm was used to visualize any potential topological features in the dataset. Before passing to the dimension reduction algorithms the data was sorted by Appendicular Lean Mass (ALM), so that any clustering based on ALM would be evident by the gradient of colors. As seen in the previous subsection (\ref{subsec:umap_results}), the ALM maintained a gradient of low (purple) and high (yellow) values when projected using the UMAP visualization tool. This pattern is not seen in the MAPPER projections. In the MAPPER projections, there is a more homogeneous pattern of the coloring of the vertices. This implies that the ALM does not distinguish different groups of clusters in the different embeddings. If the ALM did distinguish different groups, there would be a clear distinction between clusters with high and low ALM. By using the MAPPER algorithm, possible clustering patterns could be taken into consideration. Because each of the subfigures in Figure \ref{fig:mapper} have little interesting clustering patterns, further cluster analysis and topological data analysis were not performed on the data. For more details on the MAPPER and how it was implemented in more detail, see the appendix \ref{app:mapper_umap}.

\begin{figure}[H]
    \centering
    {%
    \setlength{\fboxsep}{0pt}%
    \setlength{\fboxrule}{0.5pt}%
    \fbox{\includegraphics[width=0.3\textwidth, height = 0.3\textwidth]{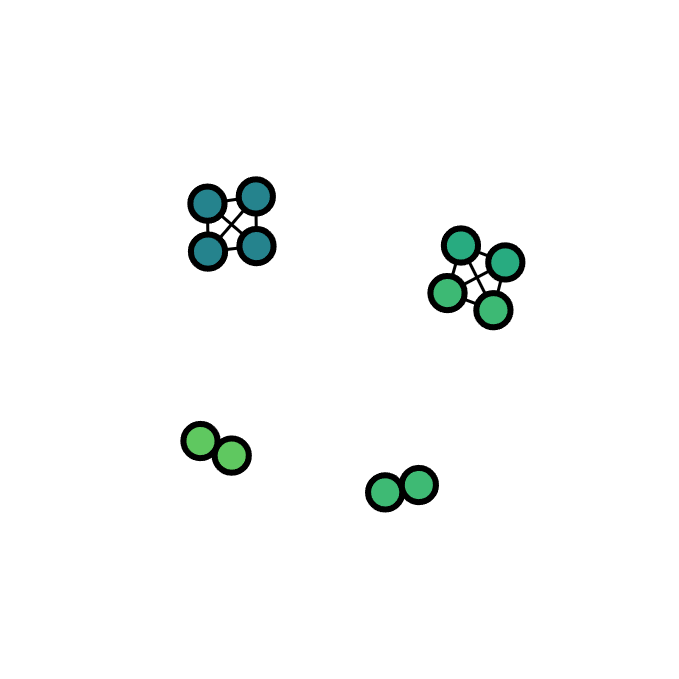}}%
    }%
    {%
    \setlength{\fboxsep}{0pt}%
    \setlength{\fboxrule}{0.5pt}%
    \fbox{\includegraphics[width=0.3\textwidth, height = 0.3\textwidth]{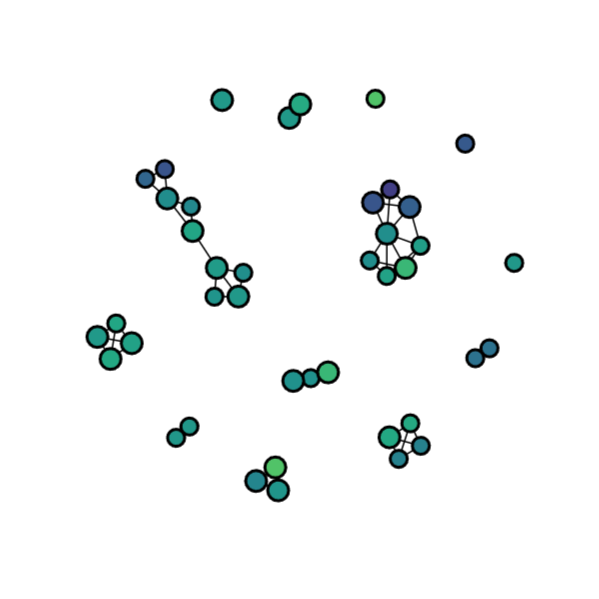}}%
    }%
    {%
    \setlength{\fboxsep}{0pt}%
    \setlength{\fboxrule}{0.5pt}%
    \fbox{\includegraphics[width=0.3\textwidth, height = 0.3\textwidth]{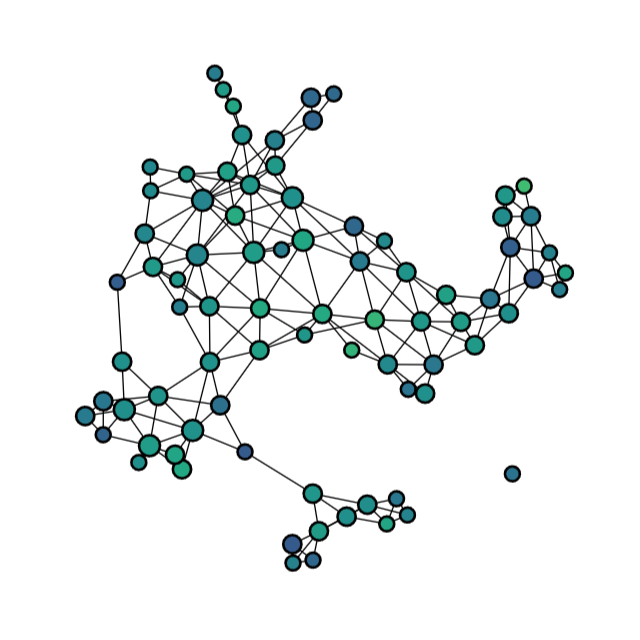}}%
    }%
    \caption{Exploratory MAPPER graphs for the standardized raw-feature
representation (left), VAE posterior-mean representation (middle), and
GMVAE expected latent representation (right). Each node represents a
cluster within a pullback-cover element, and an edge indicates that two
clusters share at least one participant. Node color denotes the mean
standardized ALM of the participants in the cluster, using a common
color scale across panels; node size denotes the relative size of the cluster. The MAPPER graphs are descriptive and are distinct from the participant-similarity graphs used by the prediction model.}
    \label{fig:Penn_data_mapper}
\end{figure}

\subsection{Feature--Target Correlation Structure}
\label{subsec:correlation_structure}

\begin{figure}[!t]
  \centering
  \includegraphics[width=0.65\linewidth]
  {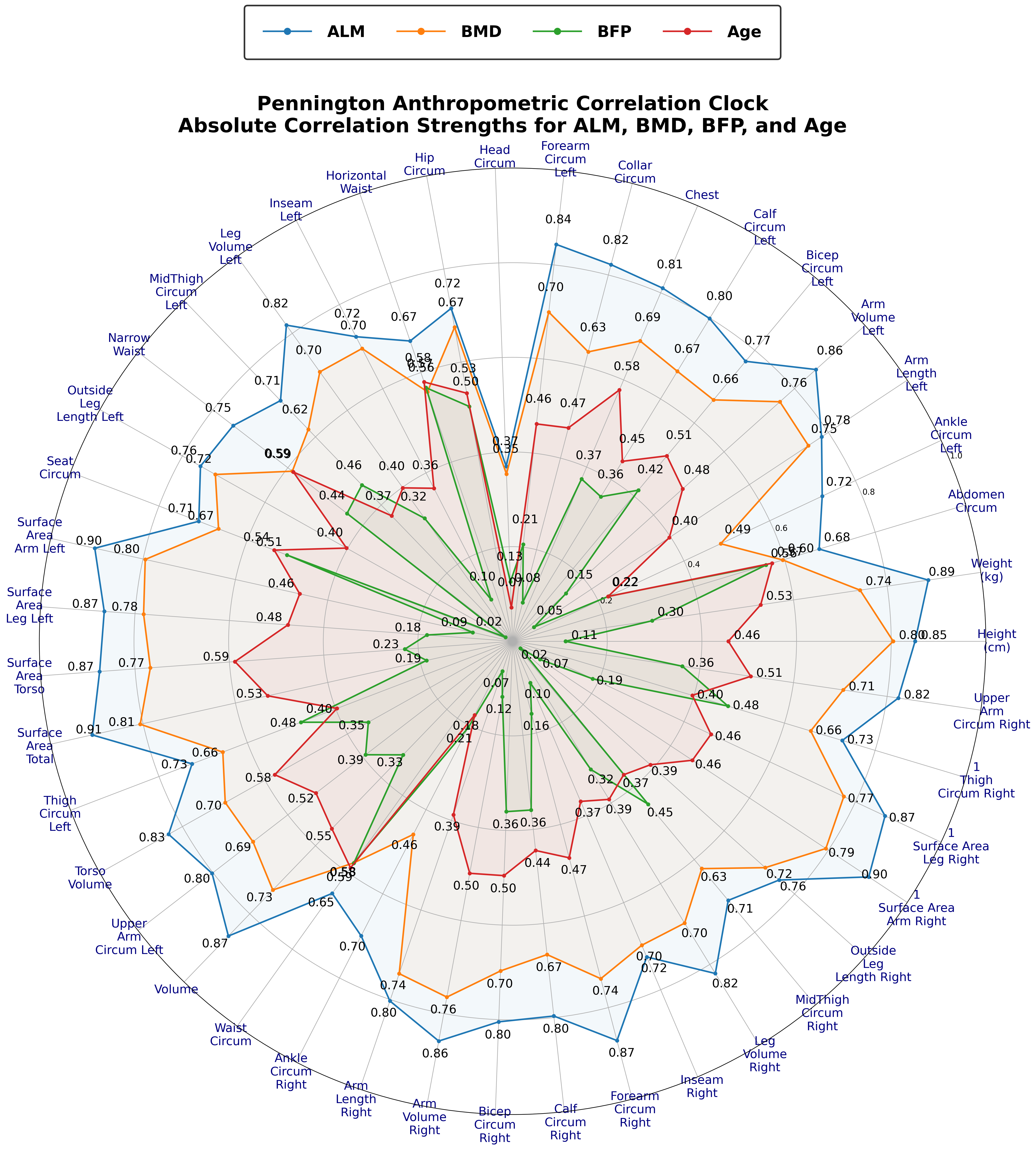}
  \caption{Absolute Pearson correlations between the non-invasive
  anthropometric predictors and ALM, BMD, BFP, and Age in the combined
  Pennington cohort. The figure is a descriptive full-cohort summary and is
  separate from the training-fold correlations used to construct the
  correlation-weighted prediction graphs.}
  \label{fig:clock}
\end{figure}

Figure~\ref{fig:clock} summarizes the absolute Pearson correlations between
the anthropometric predictors and ALM, BMD, BFP, and Age in the combined
cohort. Each spoke represents one predictor, and the radial coordinate gives
the absolute magnitude of its marginal correlation with the corresponding
target. Because absolute values are displayed, the figure represents the
strength, but not the direction, of each association.

The four outcomes exhibit distinct correlation profiles. ALM has broadly
strong associations with measures of overall and regional body size,
including weight, surface area, limb volume, and several circumference
measurements. BMD is also associated with multiple size-, length-, and
surface-area-related predictors, although its correlations are generally
weaker than those observed for ALM. In contrast, the strongest BFP
associations are concentrated among waist, abdominal, hip, seat, and thigh
circumference measurements, while many body-length variables have weaker
associations with BFP. Age shows moderate associations primarily with waist-
and torso-related measurements.

These differing profiles provide descriptive motivation for constructing
outcome-specific participant-similarity metrics: measurements that are
informative for one outcome need not be equally informative for another.
The figure does not, however, demonstrate that correlation weighting improves
prediction. That question is evaluated through the held-out comparison
between the ordinary Euclidean and correlation-weighted graph models reported
below.

The displayed correlations are marginal associations and should not be
interpreted as causal effects or as complete measures of predictive
importance. The predictors are themselves correlated and may therefore
contain overlapping information. Moreover, because Figure~\ref{fig:clock}
uses the combined cohort, its correlations may reflect both within-sex
associations and between-sex differences. The full numerical correlation
values are reported in Table~\ref{corr_all} in the Supplementary Material.

Figure~\ref{fig:clock} is used only to characterize the dataset. For every
correlation-weighted prediction model, feature weights were recomputed from
the training observations within each cross-validation split and within the
corresponding cohort, as described in Section~\ref{sec:pgnn}. No validation
or test target values were used to estimate these weights or construct the
prediction graphs.

\subsection{Results for ALM, BMD, BFP using \texorpdfstring{RMSE\textsubscript{relativepct}}{RMSE relativepct}}
\newcommand{\onestd}[1]{{\footnotesize\textsubscript{\textcolor{black!50}{$\pm$#1}}}} 



The prediction outcomes are Appendicular Lean Mass (ALM), Bone Mineral Density (BMD), and Body Fat Percentage (BFP). Seven configurations were evaluated separately in the male, female, and combined cohorts, and the tables report the lowest relative root-mean-square error, \(\mathrm{RMSE}_{\mathrm{relativepct}}\), obtained by each configuration. For \(p\)SADE-GNR, candidate \(p\)-values spanning \(2\) to \(10^{6}\) were evaluated, and the lowest-error result was retained. As defined in Section~\ref{para:summary}, \(L^2\) denotes Euclidean graph construction, \(cc\) denotes the correlation-weighted distance, and VAE and GMVAE denote latent representations used before graph construction.

Table~\ref{tab:avgkfold_results_relative_RMSE} also compares our results with the strongest benchmarks reported by \cite{gyaneshwar_et_al} and with the \(p\)-Laplacian GNN of \cite{pgnn}. The comparison with \cite{gyaneshwar_et_al} is particularly direct because both studies use the identical preprocessed dataset, the same ALM, BMD, and BFP outcomes, the same male, female, and combined cohorts, the same \(80{:}20\) training-to-test ratio, and the same \(\mathrm{RMSE}_{\mathrm{relativepct}}\) metric. Their results therefore provide closely aligned within-dataset benchmarks rather than estimates transferred from a different cohort or dataset. That study evaluated six supervised regression families—linear and polynomial regression with traditional, Lasso, ridge, and Bayesian variants; SVR; LSSVR; random forest; XGBoost; and a multilayer perceptron—together with two game-theoretic \(p\)-Laplacian regression variants. Because either SVR or LSSVR achieved the lowest error in all nine outcome--cohort combinations, we report the better of the two for each corresponding task.

The \(p\)-Laplacian GNN of \cite{pgnn} was originally formulated for classification with a softmax output; we adapted it to regression by replacing the softmax layer with a linear output. We evaluated every integer \(p\in\{1,\ldots,10\}\) and retained the lowest-error result, as values \(p>10\) were numerically unstable in our experiments.

\paragraph{Evaluation Metrics.}
\begin{definition}[Relative Root Mean Squared Error]
\begin{equation}
\mathrm{RMSE}_{\mathrm{relativepct}}
=100\times
\frac{
\sqrt{
\frac{1}{n}
\sum_{i=1}^{n}
\left( y_i - \hat{y}_i \right)^2
}}
{\sqrt{\frac{1}{n}\sum_{i=1}^{n}y_i^2}}.
\end{equation}
\end{definition}
For Table \ref{tab:avgkfold_results_relative_RMSE} $\mathrm{RMSE}_{\mathrm{relativepct}}$ is used.

\begin{definition}[Root Mean Squared Error]
The Root Mean Squared Error (RMSE) on the original scale is defined as
\begin{equation}
\mathrm{RMSE}_{\mathrm{orig}}
=
\sqrt{
\frac{1}{n}
\sum_{i=1}^{n}
\left( y_i - \hat{y}_i \right)^2
}.
\end{equation}
\end{definition}
For Tables \ref{tab:combined_results_original_RMSE} and \ref{tab:combined_results_RMSE_orig_age}, $\mathrm{RMSE}_{\mathrm{orig}}$ is used.

\paragraph{Five-Fold Cross Validation.}
For each cohort, the data were randomly partitioned into five folds. In each iteration, four folds were used for model training and the remaining fold for validation. The validation fold was rotated across the five iterations so that each participant was included in the validation set once. Model performance was then summarized by the mean and standard deviation of the prediction error across the five folds.

\paragraph{Results.} To quantify the magnitude of improvement relative to the previously reported supervised benchmarks \cite{gyaneshwar_et_al}, we calculated the relative reduction in normalized RMSE as \[ \mathrm{Reduction}(\%) = 100 \frac{ \mathrm{RMSE}_{\mathrm{benchmark}} - \mathrm{RMSE}_{\mathrm{\textit{p}SADE-GNR}} }{ \mathrm{RMSE}_{\mathrm{benchmark}} }. \] For ALM, the correlation-weighted raw-feature \(p\)SADE-GNR reduced normalized RMSE by \(32.7\%\), \(8.6\%\), and \(39.1\%\) relative to the strongest reported SVR or LSSVR benchmark in the male, female, and combined cohorts, respectively. For BFP, the corresponding reductions were \(28.4\%\), \(41.7\%\), and \(34.3\%\). BMD performance was broadly comparable with the published benchmark: normalized RMSE was reduced by \(4.9\%\) and \(3.2\%\) in the male and combined cohorts, respectively, but was \(7.5\%\) higher in the female cohort. Overall, the proposed configuration obtained numerically lower normalized RMSE in eight of the nine primary target--cohort comparisons, with the largest gains observed for ALM and BFP.

\begin{table}[H]
\centering
\caption{Average 5-fold $\mathrm{RMSE}_{\mathrm{relativepct}}$ are the percentage values across ALM, BMD, and BFP prediction tasks. The underlined values in bold indicate the lowest $\mathrm{RMSE}_{\mathrm{relativepct}}$ within each target and sex group.}
\renewcommand{\arraystretch}{1.35}
\setlength{\tabcolsep}{4pt}

\resizebox{\textwidth}{!}{
\begin{tabular}{l|l|cc|cc|cc|c|c|cc|}
\hline
\rowcolor{gray!15}
\textbf{Target} 
& \textbf{Sex}
& \multicolumn{2}{c|}{\textbf{$p$SADE-GNR}}
& \multicolumn{2}{c|}{\textbf{VAE-$p$SADE-GNR}}
& \multicolumn{2}{c|}{\textbf{GMVAE-$p$SADE-GNR}}
& \textbf{Autoencoder} 
&\textbf{Results from \cite{gyaneshwar_et_al}}
&\textbf{$p$GNN}\\

\rowcolor{gray!15}
& 
& $\mathbf{L^2}$ & $\mathbf{cc\text{-}L^2}$
& $\mathbf{L^2}$ & $\mathbf{cc\text{-}L^2}$
& $\mathbf{L^2}$ & $\mathbf{cc\text{-}L^2}$
& $\mathbf{L^2}$
& 
& \\
\hline

\multirow{3}{*}{ALM}
& Male   & 6.27\onestd{0.99}  & \uuline{\textbf{4.24}}\onestd{0.34} & 7.75\onestd{0.87} & 7.66\onestd{0.88} & 7.27\onestd{1.12} & 7.23\onestd{1.38} & 8.09\onestd{0.83} & 6.30 (SVR) & 11.86\onestd{1.68} \\
& Female & 8.49\onestd{1.31} & \uuline{\textbf{5.51}}\onestd{1.19} & 9.90\onestd{0.71} & 9.95\onestd{0.58}  & 8.95\onestd{1.04} & 8.65\onestd{1.33} & 10.19\onestd{0.64} & 6.03 (LSSVR) & 12.51\onestd{1.06}\\
& Combined   & 7.53\onestd{0.68} & \uuline{\textbf{4.77}}\onestd{0.42} & 9.85\onestd{0.81}  & 9.60\onestd{1.01} & 8.08\onestd{0.84} & 7.87\onestd{0.72} & 10.58\onestd{1.03}  & 7.83 (SVR) & 13.91\onestd{0.52} \\
\hline

\multirow{3}{*}{BMD}
& Male   & 6.84\onestd{0.47} & \uuline{\textbf{6.63}}\onestd{0.64} & 7.75\onestd{1.21}  & 7.85\onestd{1.48} & 6.83\onestd{0.73} & 6.95\onestd{0.63} & 7.93\onestd{0.97} &  6.97 (SVR) & 9.32\onestd{1.39}  \\
& Female & 7.95\onestd{1.11}  & 7.58\onestd{0.78}  & 8.99\onestd{1.13}  & 8.89\onestd{1.06}  & 8.15\onestd{1.25} & 7.99\onestd{0.85} & 9.31\onestd{0.77} & \uuline{\textbf{7.05}} (LSSVR) & 9.34\onestd{0.73} \\
& Combined   & 7.59\onestd{0.34} & \uuline{\textbf{7.24}}\onestd{0.51}  & 8.60\onestd{0.36} &  8.55\onestd{0.23}& 7.51\onestd{0.38} & 7.39\onestd{0.28} & 8.63\onestd{0.66} & 7.48 (SVR) & 9.31\onestd{0.45} \\
\hline

\multirow{3}{*}{BFP}
& Male   & 12.11\onestd{0.85}  & \uuline{\textbf{8.21}}\onestd{0.96} & 16.27\onestd{2.20} & 16.13\onestd{2.69} & 11.85\onestd{0.88}  & 12.41\onestd{0.81} & 16.79\onestd{2.30} & 11.47 (LSSVR) & 21.34\onestd{1.95} \\
& Female & 9.90\onestd{0.64}  & \uuline{\textbf{6.83}}\onestd{0.56} & 12.66\onestd{1.26}  & 12.71\onestd{1.15} & 9.99\onestd{0.44} & 9.83\onestd{0.69} & 12.50\onestd{1.08} & 11.72 (SVR) & 15.34\onestd{1.66} \\
& Combined   & 11.15\onestd{0.45} & \uuline{\textbf{7.22}}\onestd{0.53} & 14.62\onestd{0.56} & 14.26\onestd{0.87} & 11.15\onestd{0.45} & 11.12\onestd{0.29}  & 15.15\onestd{0.76} & 10.99 (LSSVR) & 17.77\onestd{1.04} \\
\hline
\end{tabular}
}

\label{tab:avgkfold_results_relative_RMSE}
\end{table}

\paragraph{Clinical relevance.} The magnitude of the observed error reductions may be practically important because ALM, BMD, and BFP are used in body-composition assessment, risk stratification, and longitudinal health evaluation. Errors in these quantities may be particularly consequential for individuals whose true measurements lie near clinically relevant decision thresholds. Based on the new normalized RMSE results, the correlation-weighted raw-feature \(p\)SADE-GNR achieved approximate original-scale RMSE values of \(0.99\), \(0.90\), and \(0.95\,\mathrm{kg}\) for ALM in the male, female, and combined cohorts, respectively. The corresponding values were approximately \(2.16\), \(2.42\), and \(2.27\) percentage points for BFP, and \(0.072\), \(0.077\), and \(0.076\,\mathrm{g/cm^2}\) for BMD. These values correspond to overall ranges of \(0.90\)--\(0.99\,\mathrm{kg}\) for ALM, \(2.16\)--\(2.42\) percentage points for BFP, and \(0.072\)--\(0.077\,\mathrm{g/cm^2}\) for BMD across the evaluated cohorts. Nevertheless, the present study evaluates continuous measurement estimation rather than diagnostic accuracy. The proposed model should therefore be interpreted as a potential non-invasive screening or measurement-support tool rather than as a replacement for DXA or a standalone diagnostic system. Threshold-specific, externally validated, and prospective analyses are required before clinical deployment.

\begin{table}[H]
\centering
\caption{Average 5-fold $\mathrm{RMSE}_{\mathrm{orig}}$ values for the same $p$SADE-GNR-based models on ALM, BMD, and BFP as those reported in Table~\ref{tab:avgkfold_results_relative_RMSE}. The values are expressed in the original units of the corresponding target variables. The underlined values indicate the lowest $\mathrm{RMSE}_{\mathrm{orig}}$ within each sex group.}

\renewcommand{\arraystretch}{1.15}
\setlength{\tabcolsep}{4pt}

\resizebox{\textwidth}{!}{
\begin{tabular}{l|l|cc|cc|cc|c|cc}
\hline
\rowcolor{gray!15}
\textbf{Target} 
& \textbf{Sex}
& \multicolumn{2}{c|}{\textbf{$p$SADE-GNR}}
& \multicolumn{2}{c|}{\textbf{VAE-$p$SADE-GNR}}
& \multicolumn{2}{c|}{\textbf{GMVAE-$p$SADE-GNR}}
& \textbf{Autoencoder} 
& \textbf{$p$GNN} \\

\rowcolor{gray!15}
& 
& $\mathbf{L^2}$ & $\mathbf{cc\text{-}L^2}$
& $\mathbf{L^2}$ & $\mathbf{cc\text{-}L^2}$
& $\mathbf{L^2}$ & $\mathbf{cc\text{-}L^2}$
& $\mathbf{L^2}$ 
& \\
\hline

\multirow{3}{*}{ALM ($kg$)}
& Male     & 1.46\onestd{0.26} & \uline{0.99}\onestd{0.10}  & 1.81\onestd{0.24} & 1.79\onestd{0.25} & 1.69\onestd{0.28} & 1.68\onestd{0.34} & 1.88\onestd{0.20} & 2.76\onestd{0.43}\\
& Female   & 1.39\onestd{0.22} & \uuline{0.90}\onestd{0.20} & 1.62\onestd{0.13} & 1.62\onestd{0.09} & 1.46\onestd{0.20} & 1.41\onestd{0.24} & 1.66\onestd{0.09} & 2.04\onestd{0.18} \\
& Combined & 1.50\onestd{0.13} & \uuline{0.95}\onestd{0.07} & 1.96\onestd{0.17} & 1.91\onestd{0.20} & 1.61\onestd{0.14} & 1.57\onestd{0.11} & 2.11\onestd{0.20} & 2.77\onestd{0.13} \\
\hline

\multirow{3}{*}{BMD ($\mathrm{gm/cm^2}$)}
& Male     & \uuline{0.07}\onestd{0.01} & \uuline{0.07}\onestd{0.01} & 0.08\onestd{0.01} & 0.09\onestd{0.02} & \uuline{0.07}\onestd{0.01} & 0.08\onestd{0.01} & 0.09\onestd{0.01} & 0.10\onestd{0.01} \\
& Female   & \uuline{0.08}\onestd{0.01} & \uuline{0.08}\onestd{0.01} & 0.09\onestd{0.01} & 0.09\onestd{0.01} & \uuline{0.08}\onestd{0.01} & \uuline{0.08}\onestd{0.01} & 0.09\onestd{0.01} & 0.09\onestd{0.01} \\
& Combined & \uuline{0.08}\onestd{0.00} & \uuline{0.08}\onestd{0.01} & 0.09\onestd{0.00} & 0.09\onestd{0.00} & \uuline{0.08}\onestd{0.00} & \uuline{0.08}\onestd{0.00} & 0.09\onestd{0.01} & 0.10\onestd{0.00} \\
\hline

\multirow{3}{*}{BFP ($\%$) }
& Male     & 3.18\onestd{0.19} & \uuline{2.16}\onestd{0.24} & 4.28\onestd{0.54} & 4.24\onestd{0.66} & 3.12\onestd{0.20} & 3.27\onestd{0.18} & 4.43\onestd{0.66} & 5.62\onestd{0.54}\\
& Female   & 3.51\onestd{0.21} & \uuline{2.42}\onestd{0.24} & 4.48\onestd{0.34} & 4.50\onestd{0.29} & 3.55\onestd{0.22} & 3.49\onestd{0.23} & 4.43\onestd{0.29} & 5.43\onestd{0.50}\\
& Combined & 3.51\onestd{0.23} & \uuline{2.27}\onestd{0.17} & 4.60\onestd{0.21} & 4.49\onestd{0.31} & 3.51\onestd{0.18} & 3.50\onestd{0.12} & 4.77\onestd{0.30} & 5.58\onestd{0.19}\\
\hline

\end{tabular}
}

\label{tab:combined_results_original_RMSE}
\end{table}

\subsection{Results for Age using \texorpdfstring{RMSE\textsubscript{orig}}{RMSE orig}}


\begin{table}[H]
\centering
\caption{Average 5 fold $\mathrm{RMSE}_{orig}$ values across Age prediction tasks. The underlined values indicate the lowest RMSE within each target and sex group.}

\renewcommand{\arraystretch}{1.15}
\setlength{\tabcolsep}{4pt}

\resizebox{\textwidth}{!}{
\begin{tabular}{l|l|cc|cc|cc|c|c}
\hline
\rowcolor{gray!15}
\textbf{Target} 
& \textbf{Sex}
& \multicolumn{2}{c|}{\textbf{$p$SADE-GNR}}
& \multicolumn{2}{c|}{\textbf{VAE-$p$SADE-GNR}}
& \multicolumn{2}{c|}{\textbf{GMVAE-$p$SADE-GNR}}
& \textbf{Autoencoder} 
& \textbf{$p$GNN} \\

\rowcolor{gray!15}
& 
& $\mathbf{L^2}$ & $\mathbf{cc\text{-}L^2}$
& $\mathbf{L^2}$ & $\mathbf{cc\text{-}L^2}$
& $\mathbf{L^2}$ & $\mathbf{cc\text{-}L^2}$
& $\mathbf{L^2}$ 
& \\
\hline
 
\multirow{3}{*}{\shortstack{Age (years) \\ with ALM, BMD, BFP}}
& Male   & 9.84\onestd{1.54} & 9.08\onestd{1.41} & 11.44\onestd{2.28} & 10.94\onestd{2.03} & 8.01\onestd{2.00} & \uuline{7.99}\onestd{1.96} & 10.02\onestd{1.60} & 13.75\onestd{0.94} \\
& Female & 11.58\onestd{0.95} & 11.73\onestd{1.28} & 14.67 \onestd{1.13}& 15.42\onestd{1.77} & 11.34\onestd{1.42} & \uuline{10.82}\onestd{1.62} & 14.75\onestd{1.44} & 16.72\onestd{1.93} \\
& Combined   & 11.13\onestd{1.00}  & 11.26\onestd{0.88} & 13.50\onestd{0.75}  & 13.29\onestd{0.65}  & 11.12\onestd{1.00} & \uuline{10.90}\onestd{1.17} & 13.82\onestd{1.74}& 15.25\onestd{0.93}\\
\hline

\multirow{3}{*}{\shortstack{Age (years) \\ without ALM, BMD, BFP }}
& Male   & 9.84\onestd{1.55}  & \uuline{9.08}\onestd{1.44}  & 11.57\onestd{2.21} & 10.93\onestd{1.84} & 9.58\onestd{2.52}  & 9.44\onestd{2.57} & 10.48\onestd{1.30} & 13.69\onestd{1.01} \\
& Female & 11.53\onestd{1.05} & 11.73\onestd{1.26} & 14.69\onestd{1.24} & 15.45\onestd{1.10} & 11.36\onestd{1.72} & \uuline{11.17}\onestd{1.11} & 14.67\onestd{1.39} & 16.38\onestd{2.29}\\
& Combined   & \uuline{11.07}\onestd{1.06} & 11.26\onestd{0.90} & 13.46\onestd{0.71} & 13.29 \onestd{0.74}& 11.39\onestd{1.45}  & 11.54\onestd{0.94} & 13.45\onestd{1.13}& 15.61\onestd{1.46} \\
\hline
\end{tabular}
}

\label{tab:combined_results_RMSE_orig_age}
\end{table}
Age prediction (Table \ref{tab:combined_results_RMSE_orig_age}) was evaluated as an exploratory task under two predictor settings: with and without the DXA-derived ALM, BMD, and BFP measurements. 
When ALM, BMD, and BFP were included, the correlation-weighted GMVAE-pSADE-GNR achieved the lowest RMSE in all three cohorts, with values of $7.99$, $10.82$, and $10.90$ years for the Male, Female, and Combined cohorts, respectively. 
When these variables were excluded, the best-performing configuration varied across cohorts. 
The correlation-weighted raw-feature $p$SADE-GNR achieved the lowest RMSE for Males ($9.08$ years), the correlation-weighted GMVAE-$p$SADE-GNR achieved the lowest RMSE for Females ($11.17$ years), and the raw-feature $L^2$-$p$SADE-GNR achieved the lowest RMSE for the Combined cohort ($11.07$ years). 
These results indicate that including ALM, BMD, and BFP was particularly beneficial for the GMVAE-based Age models, whereas the raw-feature $p$SADE-GNR results changed very little between the two predictor settings.

\section{Discussion}\label{sec:discussion} 

The central result of this study is that the target-aware $p$SADE-GNR constructed from the original anthropometric measurements substantially improved the prediction of ALM, BMD, and BFP. The correlation-weighted raw-feature $p$SADE-GNR achieved the lowest cross-validated RMSE among all evaluated configurations in each of the nine primary target--cohort settings. Relative to the corresponding unweighted $p$SADE-GNR, target-aware graph construction reduced relative RMSE by approximately $7\%$--$40\%$, with the largest improvements observed for ALM and BFP (Table~\ref{tab:avgkfold_results_relative_RMSE}). The same model also obtained numerically lower normalized RMSE than the strongest published SVR or LSSVR reference value in eight of the nine comparisons \cite{gyaneshwar_et_al}. In original units, its RMSE was approximately $0.90$--$0.99\,\mathrm{kg}$ for ALM, $2.16$--$2.42$ percentage points for BFP, and $0.07$--$0.08\,\mathrm{gm/cm^2}$ for BMD. These results provide strong evidence that participant-similarity modeling through the proposed target-aware $p$SADE-GNR is effective for estimating body-composition outcomes from non-invasive anthropometric measurements.

The consistent improvement produced by correlation weighting shows that graph construction is a decisive component of the proposed framework. A standard Euclidean graph treats all standardized measurements equally, whereas the target-aware distance emphasizes features that are more strongly associated with the outcome in the training data. Because ALM, BMD, and BFP exhibit different feature--target correlation profiles (Figure~\ref{fig:clock}), a single fixed definition of participant similarity is not equally informative for all three prediction tasks. The lower RMSE of the correlation-weighted model in every primary target--cohort comparison demonstrates that adapting the graph to the target provides a clear predictive advantage. Although the correlation weights are marginal associations rather than causal or conditional importance measures, they provide a simple, transparent, and empirically effective mechanism for constructing outcome-specific participant graphs.

In contrast, GMVAE was beneficial for exploratory age prediction when
DXA-derived ALM, BMD, and BFP were included as predictors. The
GMVAE-cc-\(L^2\)--\(p\)SADE-GNR achieved the lowest mean RMSE in the male,
female, and combined cohorts (\(7.99\), \(10.82\), and \(10.90\) years,
respectively) and was best in four of the six age cohort--setting
comparisons. Its advantage was less consistent in the anthropometry-only
setting, where it remained best only for females. These results suggest
that GMVAE has a task-dependent benefit for modeling age from combined
anthropometric and body-composition information, although it did not
improve the primary ALM, BMD, and BFP predictions. Because no formal
paired significance analysis was performed, these comparisons should be
interpreted cautiously.

Several limitations remain. First of all, the present experiments establish the performance of the complete target-aware $p$SADE-GNR pipeline, but they do not separately isolate the contributions of graph propagation and nonlinear $p$-Laplacian diffusion through the graph constructed here. Secondly, we are still unable to give an explanation why our method gives a worse result than LSSVR for the female BMD predictions. In addition, the data were obtained from a single research center and evaluated through internal five-fold cross-validation, without external validation or individual prediction intervals. Future work should therefore examine independent and prospective cohorts, compare the proposed model directly with $p$-Laplacian diffusion baseline and other regression methods on independent datasets, quantify predictive uncertainty, and improve the efficiency of graph construction and hyperparameter selection. Integration with automated anthropometric extraction from 3D body scans is another natural direction. Overall, the findings support target-aware $p$SADE-GNR modeling as a promising approach for non-invasive body-composition estimation and show that its principal advantage arises from outcome-specific participant graph construction rather than from generic lower-dimensional representation learning.

\subsection*{Acknowledgements}
The authors are especially grateful to Professor Peter Wolenski for collaboratively initiating this project and for many valuable discussions throughout its development. 

 This work was supported in part by National Science Foundation Award No.2407839 and by the National Institutes of Health through Nutrition Obesity Research Center grants P30DK072476 (Pennington/Louisiana) and P30DK040561 (Harvard), as well as research grant R01DK109008 (Shape Up! Adults).

\bibliography{main_references_revise} \label{sec:References}
\bibliographystyle{plain}
\bigskip




\newpage

\appendix
\section{Supplementary Materials}\label{app:mapper_umap}

 \subsection{Proofs of Theoretical Results}
\subsubsection{Proof of
Proposition~\ref{prop:target_aware_linear_guarantee}}
\label{app:target_aware_linear_guarantee}

\begin{proof}
Let

$$
S=\{\ell:\beta_\ell\ne0\},
\qquad
A(b)=\sum_{\ell\in S}\frac{\beta_\ell^2}{b_\ell},
$$

where \(A(b)=\infty\) if \(b_\ell=0\) for some \(\ell\in S\).
In that case, \(K(b)=A(b)=\infty\) by convention. Hence, assume
\(b_\ell>0\) for all \(\ell\in S\).

For every \(\delta\in\mathbb{R}^m\) with
\(\sum_{\ell=1}^m b_\ell\delta_\ell^2>0\), Cauchy--Schwarz gives

$$
|\beta^\top\delta|^2
\leq
A(b)\sum_{\ell\in S}b_\ell\delta_\ell^2
\leq
A(b)\sum_{\ell=1}^m b_\ell\delta_\ell^2.
$$

Thus \(K(b)^2\leq A(b)\). Equality is attained by taking
\(\delta_\ell=\beta_\ell/b_\ell\) for \(\ell\in S\) and
\(\delta_\ell=0\) otherwise. Therefore,

$$
K(b)^2=A(b)
=\sum_{\ell:\beta_\ell\ne0}\frac{\beta_\ell^2}{b_\ell}.
$$

Applying Cauchy--Schwarz once more,

$$
\|\beta\|_1^2
\leq
A(b)\sum_{\ell\in S}b_\ell
\leq A(b)=K(b)^2,
$$

because \(\sum_{\ell\in S}b_\ell\leq1\). Equality holds precisely when
all weight is supported on \(S\) and
\(b_\ell\propto|\beta_\ell|\) on \(S\). The simplex constraint then
uniquely determines

$$
b_\ell^\star
=
\frac{|\beta_\ell|}{\|\beta\|_1},
\qquad
K(b^\star)=\|\beta\|_1.
$$

Since \(|\beta_\ell|=c|\rho_\ell|\),

$$
b_\ell^\star
=
\frac{|\rho_\ell|}{\|\rho\|_1}.
$$

For \(b_\ell^{\mathrm{Euc}}=1/m\),
\(K(b^{\mathrm{Euc}})=\sqrt m\,\|\beta\|_2\). Hence

$$
\frac{K(b^\star)}{K(b^{\mathrm{Euc}})}
=
\frac{\|\rho\|_1}{\sqrt m\,\|\rho\|_2}
\leq1,
$$

by Cauchy--Schwarz. Equality holds if and only if all
\(|\rho_\ell|\) are equal.
\end{proof}

\subsubsection{Proof of Theorem \ref{thm:shrinking}}\label{proof:shrinking}
\begin{proof}
Let $g=\nabla\E_p(H)$ and $\delta=\delta(H)>0$. Recall that
\begin{equation}
(\nabla\E_p(H))_i
=
\sum_{j\in\mathcal N(i)}
w_{ij}\lVert h_i-h_j\rVert_2^{p-2}(h_i-h_j).
\label{eq:gradient}
\end{equation}

From
\eqref{eq:gradient},
\begin{equation}
\lVert g_i\rVert_2
\le
\sum_j w_{ij}\delta^{p-1}
\le d_w\delta^{p-1}.
\label{eq:node-gradient-bound}
\end{equation}
Put $\tau=\mu\delta^{2-p}$ and consider the Euler segment
$H_s=H-s\tau g$, $0\le s\le1$.  For every edge $\{i,j\}$,
\begin{align}
\lVert (H_s)_i-(H_s)_j\rVert_2
&\le
\lVert h_i-h_j\rVert_2
+s\tau(\lVert g_i\rVert_2+\lVert g_j\rVert_2) \\
&\le
\delta(1+2\mu d_w)
=R_\mu\delta.
\label{eq:segment-edge-bound}
\end{align}

For $\varphi(z)=\lVert z\rVert_2^p/p$, the Hessian has operator norm at
most $(p-1)\lVert z\rVert_2^{p-2}$.  Orienting the edges and using the
incidence matrix, the Hessian quadratic form of $\E_p$ therefore obeys, for
every $U\in\R^{n\times r}$,
\begin{align}
\ip{U}{\nabla^2\E_p(H_s)U}
&\le
(p-1)(R_\mu\delta)^{p-2}
\sum_{\{i,j\}\in E}w_{ij}\lVert u_i-u_j\rVert_2^2 \\
&\le
(p-1)(R_\mu\delta)^{p-2}\Lambda_w\lVert U\rVert_F^2.
\label{eq:hessian-bound}
\end{align}
Taylor's theorem along the segment, with
$L=(p-1)(R_\mu\delta)^{p-2}\Lambda_w$, now yields
\begin{align}
\E_p(H-\tau g)
&\le
\E_p(H)-\tau\lVert g\rVert_F^2
+\frac{L\tau^2}{2}\lVert g\rVert_F^2 \\
&=
\E_p(H)-
\mu\delta^{2-p}
\left[1-\frac{\mu(p-1)\Lambda_wR_\mu^{p-2}}{2}\right]
\lVert g\rVert_F^2,
\end{align}
which is \eqref{eq:quantitative-descent}.  Finally,
$\ip{\nabla\E_p(H)}{H}=p\E_p(H)$, so a nonconstant state on a
positive-weight component has nonzero gradient.
\end{proof}

\subsection{MAPPER}\label{sec:mapper}
\begin{figure}[H]
    \centering
    \includegraphics[width = 0.35\linewidth]{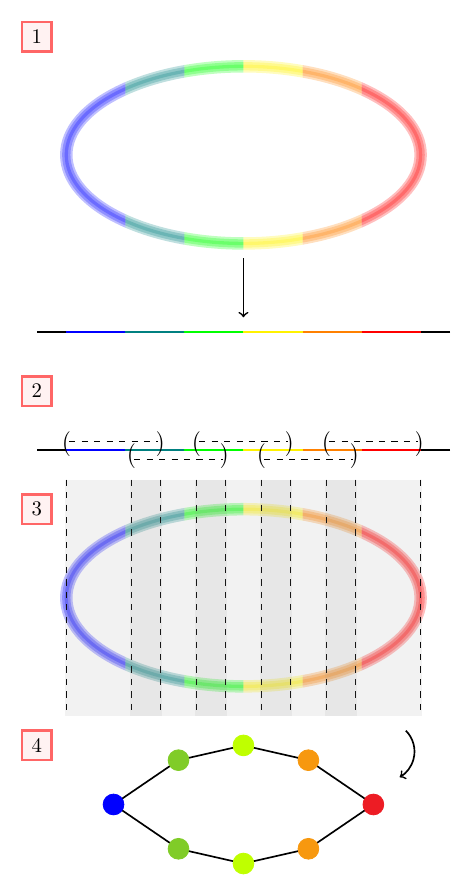}
    \caption{Toy illustration of the MAPPER algorithm. A filter function is chosen to project the dataset, overlapping regions are formed in the projected space, clustering is performed within each preimage, and clusters with shared data points are connected in the resulting graph.}
    \label{fig:mapper}
\end{figure}
The MAPPER algorithm is a method for projecting and visualizing high-dimensional data \cite{Mapper}. Its purpose is to summarize clustering patterns while retaining information about the topological structure of the dataset. The algorithm begins by choosing a \textit{filter function} that maps the data to a lower-dimensional space. The range of this filter function is then covered by overlapping intervals or regions, and a clustering algorithm is applied within the pre-image of each region. Principal Component Analysis was used here to construct the filter function using the primary and secondary axes of the data. 

The output of MAPPER is a graph. Vertices represent clusters, and two vertices are connected by an edge if the corresponding clusters share at least one data point. Vertex size can also be scaled according to the number of samples in each cluster, giving a compact visualization of both local density and overlap structure. Figure~\ref{fig:mapper} illustrates this construction using a simple toy example. In this project, the MAPPER algorithm was used as a visualization tool to see any significant topological features before and after the different dimension reduction algorithms were used.

\subsection{Uniform Manifold Approximation and Projection (UMAP)}\label{app:umap_background}

Uniform Manifold Approximation and Projection (UMAP) is a nonlinear manifold learning method for constructing low-dimensional embeddings of high-dimensional data \cite{mcinnes2020umap}. UMAP first builds a weighted neighborhood graph in the original feature space and then optimizes a low-dimensional embedding whose neighborhood structure approximates that of the original data. The method is based on local manifold approximations, fuzzy simplicial sets, and a cross-entropy objective between high-dimensional and low-dimensional topological representations.

Two important UMAP parameters are $n_{\mathrm{neighbors}}$ and \texttt{min\_dist}. The parameter $n_{\mathrm{neighbors}}$ controls the size of the local neighborhoods used to approximate the data manifold. Smaller values emphasize more local structure, while larger values incorporate broader global structure. The parameter \texttt{min\_dist} controls how tightly points are allowed to pack together in the low-dimensional embedding; smaller values allow denser clusters, while larger values produce more spread-out embeddings. In this paper, UMAP is used only as an exploratory visualization tool. The experiment-specific UMAP settings and ALM-colored visualizations are reported in Section \ref{subsec:umap_results}.
\subsection{Computational Costs}
\begin{table}[H]
\centering
\caption{Average computational cost per cross-validation fold for the $p$SADE-GNR-based methods for the combined three datasets: Male, Female, and Combined. The reported runtime represents the mean wall-clock time required to complete one fold of the complete hyperparameter search across the Male, Female, and Combined datasets. All proposed methods were executed on the same CUDA GPU computing environment, while the SVR and LSSVR runtimes are taken from \cite{gyaneshwar_et_al}, where experiments were performed on a CPU.}
\label{tab:computational_costkfold}
\renewcommand{\arraystretch}{1.35}
\setlength{\tabcolsep}{4pt}

\resizebox{\textwidth}{!}{
\begin{tabular}{l|cc|cc|cc|cc|cc}
 \hline
\rowcolor{gray!15}
\textbf{Runtime Measure}
& \multicolumn{2}{c|}{\textbf{$p$SADE-GNR}}
& \multicolumn{2}{c|}{\textbf{VAE-$p$SADE-GNR}}
& \multicolumn{2}{c|}{\textbf{GMVAE-$p$SADE-GNR}}
& \multicolumn{2}{c|}{\textbf{Autoencoder}}
& \multicolumn{2}{c}{\textbf{Results from \cite{gyaneshwar_et_al}}}\\

\rowcolor{gray!15}
&$\mathbf{L^2}$ & $\mathbf{cc\text{-}L^2}$
& $\mathbf{L^2}$ & $\mathbf{cc\text{-}L^2}$
& $\mathbf{L^2}$ & $\mathbf{cc\text{-}L^2}$
& $\mathbf{L^2}$ & 
& \textbf{SVR} & \textbf{LSSVR} \\
 \hline

Device 
& CUDA GPU & CUDA GPU 
& CUDA GPU & CUDA GPU 
& CUDA GPU & CUDA GPU 
& CUDA GPU &
& CPU & CPU \\

Total Runtime
& 1h 26m 22s & 10h 45m 35s
& 2h 07m 07s & 13h 55m 50s
& 1h 27m 08s & 11h 06m 52s
& 1h 6m &
& 1m 50s & 52.2s \\

Runtime (s)
& 5122.68s & 38735.39s
& 7451.77s & 49670.56s
& 5228.16s  & 39558.12s 
& 3665.96 &
& 110s & 52.2s \\
\hline

\end{tabular}
}
\end{table}
\newpage
\subsection{Biomarkers}

\begin{table}[ht]
    \centering
    \begin{tabular}{|l|l|}
        \hline
        \rowcolor{gray!15} Biomarker & Unit of Measure\\
         \hline
        Age & Years \\
        Gender & Male/Female \\
        ALM & Kilograms (kg) \\
        BMD & Grams / Square Centimeters (gm/cm$^2$)\\
        BFP & Percentage (\%) \\
        Height & Centimeters (cm) \\
        Weight & Kilograms (kg) \\
        Abdomen Circumference & Centimeters (cm) \\
        Ankle Circumference (left/right) & Centimeters (cm) \\
        Arm Length & Centimeters (cm) \\
        Arm Volume (left/right) & Cubic Centimeters (cm$^3$) \\
        Bicep Circumference (left/right) & Centimeters (cm) \\
        Calf Circumference (left/right) & Centimeters (cm) \\
        Chest & Centimeters (cm) \\
        Collar Circumference & Centimeters (cm) \\
        Forearm Circumference (left/right) & Centimeters (cm) \\
        Head Circumference & Centimeters (cm) \\
        Hip Circumference & Centimeters (cm) \\
        Horizontal Waist & Centimeters (cm) \\
        Inseam (left/right) & Centimeters (cm) \\
        Leg Volume (left/right)  & Cubic Centimeters (cm$^3$) \\
        Mid Thigh Circumference (left/right) & Centimeters (cm) \\
        Narrow Waist & Centimeters (cm) \\
        Outside Leg Length (left/right) & Centimeters (cm) \\
        Seat Circumference & Centimeters (cm) \\
        Surface Area Arm (left/right) & Square Centimeters (cm$^2$) \\
        Surface Area Leg (left/right) & Square Centimeters (cm$^2$) \\
        Surface Area Torso & Square Centimeters (cm$^2$) \\
        Surface Area Total & Square Centimeters (cm$^2$) \\
        Thigh Circumference (left/right) & Centimeters (cm) \\
        Torso Volume & Cubic Centimeters (cm$^3$) \\
        Upper Arm Circumference (left/right) & Centimeters (cm) \\
        Volume & Cubic Centimeters (cm$^3$) \\
        Waist Circumference & Centimeters (cm) \\
        \hline
    \end{tabular}
    \caption{This is a table outlining all of the biomarkers used along with the units of measure.}
    \label{tab:biomarkers_table}
\end{table}

\subsection{Correlations}
\begin{longtable}{l|r|r|r|r|}
 
\caption{Feature correlations with ALM, BMD, BFP, and Age.} \\
 \toprule
\rowcolor{gray!15}\textbf{Feature} & \textbf{ALM} & \textbf{BMD} & \textbf{BFP} & \textbf{Age} \\
 \midrule
\endfirsthead

 \toprule
\textbf{Feature} & \textbf{ALM} & \textbf{BMD} & \textbf{BFP} & \textbf{Age} \\
 \midrule
\endhead

 \midrule
\multicolumn{5}{r}{\textit{Continued on next page}} \\
\endfoot

 \bottomrule
\endlastfoot

Height (cm) & 0.8504 & 0.8043 & 0.1119 & 0.4559 \\
Weight (kg) & 0.8879 & 0.7422 & 0.2980 & 0.5299 \\
Abdomen Circumference & 0.6765 & 0.5963 & 0.5589 & 0.5727 \\
Ankle Circumference Left & 0.7226 & 0.4857 & 0.2241 & 0.2228 \\
Arm Length Left & 0.7829 & 0.7495 & 0.0543 & 0.3971 \\
Arm Volume Left & 0.8605 & 0.7584 & 0.1514 & 0.4827 \\
Bicep Circumference Left & 0.7697 & 0.6637 & 0.4154 & 0.5098 \\
Calf Circumference Left & 0.7991 & 0.6684 & 0.3579 & 0.4455 \\
Chest & 0.8108 & 0.6895 & 0.3727 & 0.5771 \\
Collar Circumference & 0.8226 & 0.6321 & 0.0848 & 0.4661 \\
Forearm Circumference Left & 0.8437 & 0.6998 & 0.2058 & 0.4621 \\
Head Circumference & 0.3699 & 0.3536 & 0.1258 & 0.0715 \\
Hip Circumference & 0.7154 & 0.6748 & 0.5046 & 0.5331 \\
Horizontal Waist & 0.6704 & 0.5571 & 0.5661 & 0.5789 \\
Inseam Left & 0.7237 & 0.6956 & 0.0991 & 0.3635 \\
Leg Volume Left & 0.8214 & 0.7000 & 0.3195 & 0.3986 \\
MidThigh Circumference Left & 0.7065 & 0.6220 & 0.4586 & 0.3682 \\
Narrow Waist & 0.7458 & 0.5880 & 0.4420 & 0.5853 \\
Outside Leg Length Left & 0.7562 & 0.7202 & 0.0171 & 0.4024 \\
Seat Circumference & 0.7099 & 0.6652 & 0.5103 & 0.5392 \\
Surface Area Arm Left & 0.9045 & 0.7952 & 0.0862 & 0.4609 \\
Surface Area Leg Left & 0.8652 & 0.7818 & 0.1815 & 0.4760 \\
Surface Area Torso & 0.8749 & 0.7675 & 0.2284 & 0.5883 \\
Surface Area Total & 0.9100 & 0.8062 & 0.1863 & 0.5300 \\
Thigh Circumference Left & 0.7252 & 0.6556 & 0.4789 & 0.3974 \\
Torso Volume & 0.8334 & 0.6963 & 0.3496 & 0.5762 \\
Upper Arm Circumference Left & 0.8023 & 0.6931 & 0.3923 & 0.5249 \\
Volume & 0.8652 & 0.7300 & 0.3334 & 0.5502 \\
Waist Circumference & 0.6550 & 0.5753 & 0.5766 & 0.5887 \\
Ankle Circumference Right & 0.6999 & 0.4589 & 0.2057 & 0.1752 \\
Arm Length Right & 0.8030 & 0.7420 & 0.0669 & 0.3874 \\
Arm Volume Right & 0.8590 & 0.7646 & 0.1193 & 0.4986 \\
Bicep Circumference Right & 0.8047 & 0.6970 & 0.3602 & 0.4955 \\
Calf Circumference Right & 0.7966 & 0.6660 & 0.3587 & 0.4447 \\
Forearm Circumference Right & 0.8721 & 0.7369 & 0.1582 & 0.4729 \\
Inseam Right & 0.7248 & 0.6979 & 0.0955 & 0.3676 \\
Leg Volume Right & 0.8224 & 0.6978 & 0.3174 & 0.3914 \\
MidThigh Circumference Right & 0.7123 & 0.6250 & 0.4479 & 0.3669 \\
Outside Leg Length Right & 0.7561 & 0.7168 & 0.0221 & 0.3910 \\
Surface Area Arm Right & 0.9024 & 0.7937 & 0.0695 & 0.4556 \\
Surface Area Leg Right & 0.8692 & 0.7730 & 0.1866 & 0.4631 \\
Thigh Circumference Right & 0.7275 & 0.6578 & 0.4753 & 0.3965 \\
Upper Arm Circumference Right & 0.8238 & 0.7062 & 0.3626 & 0.5087 \\
\label{corr_all}
\end{longtable}

\end{document}